\documentclass[11pt]{article}

\usepackage[final]{acl}

\usepackage{times}
\usepackage{latexsym}

\usepackage[T1]{fontenc}

\usepackage[utf8]{inputenc}

\usepackage{microtype}

\usepackage{inconsolata}

\usepackage{graphicx}
\usepackage{hyperref}
\usepackage{url}
\usepackage{booktabs}       
\usepackage{amsfonts}       
\usepackage{nicefrac}       
\usepackage{microtype}      
\usepackage{xcolor}         
\usepackage{subcaption}     
\usepackage{marvosym}
\usepackage{algorithm}
\usepackage{algpseudocode}
\usepackage{amsmath}
\usepackage{listings}

\usepackage{array}
\usepackage{colortbl} 
\usepackage{titletoc}
\usepackage{multicol}
\usepackage{multirow}
\usepackage{amsmath}
\usepackage{cleveref}
\usepackage{trajan}
\usepackage{arydshln}
\usepackage{enumitem}
\usepackage{wrapfig}
\usepackage{xspace}

\usepackage{makecell}
\usepackage{amsthm}
\usepackage{amsmath}
\usepackage{amssymb}

\usepackage{comment}
\usepackage{algorithm}
\usepackage{algorithmicx}
\usepackage{algpseudocode}
\definecolor{rowblue}{RGB}{238, 242, 252}

\title{Breaking the Impasse: Dual-Scale Evolutionary Policy Training \\ for Social Language Agents}

\author{%
  Minzheng Wang$^{1,2}$, 
  Run Luo$^{2}$,
  Yanbo Wang$^{1,2}$,
  Zichen Liu$^{3}$\\
  {\bf Yuqiao Tan}$^{1,2}$,
  {\bf Tao Tan}$^{4}$,
  {\bf Xu Nan}$^{1}$,
  {\bf Yinhe Zheng}$^{4}$,
  {\bf Wenji Mao}$^{1,2}$\thanks{Corresponding authors.} \\
  $^{1}$ MAIS, Institute of Automation, Chinese Academy of Sciences\\
  $^{2}$ School of Artificial Intelligence, University of Chinese Academy of Sciences\\
  $^{3}$ National University of Singapore 
  $^{4}$ Ritzz-AI\\
\Letter:~\texttt{\{wangminzheng2023, wenji.mao\}@ia.ac.cn}}

\newcommand{\method}{DEPT}
\newcommand{\fullmethod}{Dual-scale Evolutionary Policy Training}

\begin{document}
\maketitle
\begin{abstract} 
    While Reinforcement Learning with Verifiable Rewards (RLVR) has proven effective for closed-ended tasks, extending it to open-ended social language games via self-play reveals a critical issue: \textbf{evolution impasse}. Due to the vast strategy space, language agents frequently converge to homogenized behaviors, leading to deterministic match outcomes that eliminate the gradient signals necessary for policy evolution. To tackle this issue, we propose \textbf{\fullmethod~(\method)} for social language games. \method~introduces a \textit{time-scaled evolutionary perception} mechanism that detects impasse by quantifying dual-scale value baseline divergence alongside match entropy. Upon perceiving the collapse, it then activates \textit{asymmetric advantage reshaping} to dynamically modulate the optimization landscape for intervention. Thus, our method effectively restores gradient signals and enforces sustained strategic exploration. Extensive experiments on multiple social language games demonstrate that \method~outperforms strong baselines, avoiding policy degeneration and driving the continuous evolution of social language agents.
\end{abstract}

\section{Introduction}
Reinforcement Learning with Verifiable Rewards (RLVR) has emerged as a pivotal paradigm for unlocking the reasoning potential of large language models (LLMs), demonstrating remarkable capabilities in static problem-solving tasks such as mathematics and code generation~\citep{guo2025deepseek,team2025kimi,yue2025does,liu2025understanding}. However, these tasks are inherently closed-ended problem-solving, characterized by deterministic correct answers with straightforward question answering. In reality, social decision-making is considerably more complex, often situated in open-ended contexts fraught with conflicting interests and dynamic interactions~\citep{meta2022human,park2023generative,yang2025aria}. Under such dynamic social contexts, LLMs not only engage in coherent logical reasoning but also perform long-horizon planning, infer opponents' intentions in real-time, and process complex social signals such as negotiation~\citep{bianchi2024well,hua2024game}, deception~\citep{hagendorff2023deception,xu2024language}, and bargaining~\citep{xia2024measuring,abdulhai2025lmrl}. Thus, for language agents, a significant gap exists between closed-ended problem-solving capabilities and open-ended social intelligence~\citep{zhou2024sotopia,wang2025adaptive}.

\begin{figure}[!t]
\centerline{\includegraphics[width=\columnwidth]{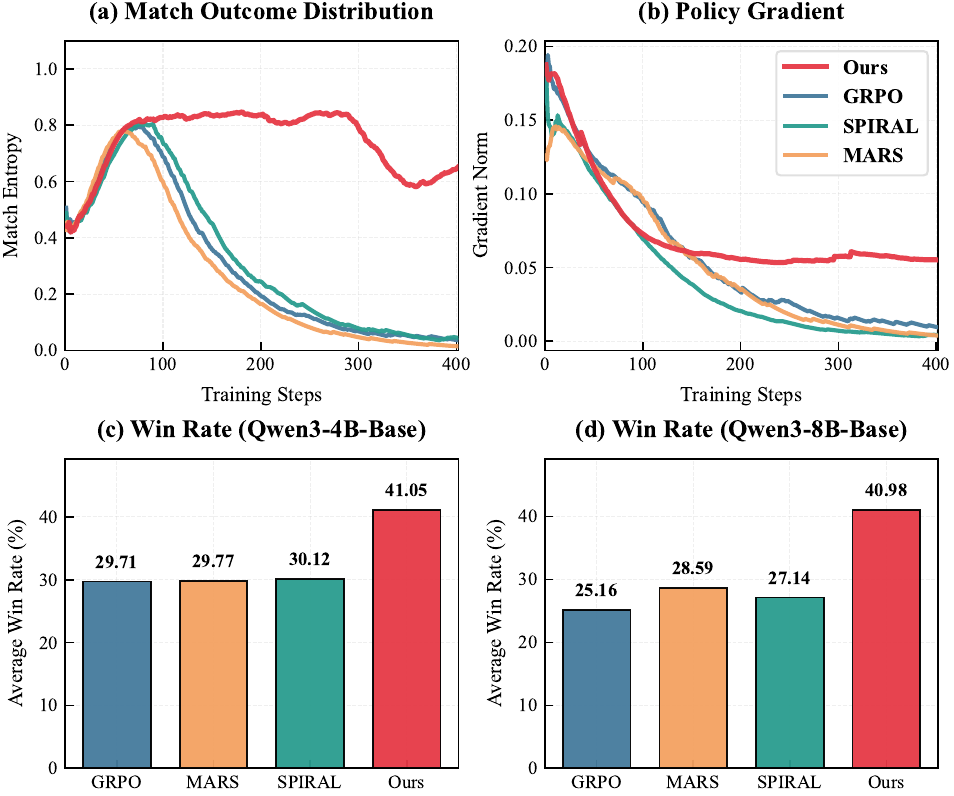}}
\caption{(a) \method~maintains high outcome entropy while baselines collapse into deterministic distributions. 
(b) \method~sustains effective gradient signals while baselines suffer from gradient vanishing. 
(c, d) \method~achieves superior performance, effectively addressing policy degeneration in open-ended social language games.}
\label{fig:intro}
\end{figure}

Social strategic games provide an ideal testbed to practise this kind of social capability in open-ended tasks: they naturally simulate dynamic social contexts and complex decision-making behaviors, while game outcomes (i.e., win/loss/draw) are objectively determined by rules—perfectly aligning with the RLVR paradigm without requiring expensive human annotations~\citep{duan2024gtbench,hua2024game,guertler2025textarena}. More importantly, this property makes self-play an effective pathway for scaling LLM capabilities~\citep{silver2018general,bansal2018emergent,chen2024self,liu2025spiral}. Using a single LLM acting as both players of the game and continuously optimizing through zero-sum competition, self-play can overcome the data scarcity bottleneck in post-training, and meanwhile establish an intrinsic automatic curriculum through opponent co-evolution of social decision-making skills.

Existing self-play RLVR methods~\citep{liu2025spiral,yuan2025mars} have achieved remarkable success in games with limited action space (e.g., Tic-Tac-Toe, Kuhn Poker), by designing role-specific advantage estimation. However, these methods struggle with open-ended social games involving free-form natural language interactions, which constitute \textit{vast action space} that often requires verbal influence on opponents in order to succeed.
Without effective guidance in such unbounded strategy space, agents easily become trapped in suboptimal solutions.
This renders the training process highly liable to \textbf{evolution impasse}, that is, the stagnant situation a cognitive agent encounters in its decision-making~\citep{laird2019soar,sumers2023cognitive}, which manifests as deterministic match outcome distributions for self-play RLVR methods. In this situation, the value baseline converges to the static expected return that causes the advantage function to approach zero, leading to the vanishing of policy gradient and locking the agent in a suboptimal state (as shown in~\Cref{fig:intro}(a)(b)). To break the impasse, it is critical to develop an effective mechanism for evolutionary perception and intervention during self-play. 

To address the above issues, we propose a \textbf{\fullmethod~(\method)} method to dynamically perceive evolutionary states and adaptively modulate the optimization landscape. Specifically, we introduce a \textit{time-scaled evolutionary perception} mechanism that maintains parallel value baselines: a fast baseline sensitive to recent fluctuations and a slow baseline anchored to long-term trends. By measuring the divergence between these baselines alongside match outcome entropy, we derive a quantifiable metric for evolution impasse.
To restore effective gradients, we propose \textit{asymmetric advantage reshaping}. By anchoring on the stable slow baseline and adaptively fusing historical value bounds from the fast baseline, this mechanism constructs a contrastive advantage signal. It effectively penalizes dominant outcomes while amplifying rare trajectories, thereby enforcing exploration. Our approach effectively intervenes in evolution
impasse and restores the gradient landscape, driving continuous strategic evolution. Experimental results across multiple social language games demonstrate that \method~outperforms strong baselines.

Our contributions are summarized as follows: (1) We propose \method~for social language games, which is the first to dynamically perceive evolution impasse and intervene in self-play RLVR. (2) Our method develops the time-scaled evolutionary perception and asymmetric advantage reshaping mechanism, which leverages dual value baselines to quantify evolutionary states and restore effective gradient signals. (3) Extensive experiments demonstrate that \method~outperforms existing SOTA methods, effectively preventing policy degeneration while driving continuous co-evolution.

\section{Preliminary}
\subsection{Task Formulation}
The social language game task is modeled as a two-player zero-sum Markov game~\citep{littman1994markov} $\mathcal{G} = (\mathcal{S}, \mathcal{A}_0, \mathcal{A}_1, \mathcal{T}, \mathcal{R})$. The state space $\mathcal{S}$ represents the interaction context, and $\mathcal{A}_0$, $\mathcal{A}_1$ denote the action spaces for each player, consisting of open-ended natural language utterances. At state $s_t$, the active player $p = t \bmod 2$ receives a partial observation $o_t$ containing the interaction history and its agent-specific private information, then samples an utterance $a_t^{(p)}$ from its policy $\pi_p(\cdot \mid o_t)$. The environment transitions to $s_{t+1}$ via $\mathcal{T}$. An episode forms a trajectory $\tau = \{(s_t, a_t^{(p)})\}_{t=0}^{T}$, concluding at terminal timestep $T$ where each agent receives a reward $R_p(\tau)$ based on its private goals. The zero-sum property requires $R_0(\tau) + R_1(\tau) = 0$, creating direct competition between agents. Typically, rewards are assigned as $+1$ for win, $-1$ for lose, and $0$ for a draw. Each agent aims to learn a policy $\pi_p$ maximizing its expected return: $\mathbb{E}_{\tau \sim \pi_0, \pi_1}[R_p(\tau)]$.

\subsection{Self-Play RLVR}
Self-play trains a single policy LLM $\pi_\theta$ to act as both players, with role conditioning achieved through system prompts specifying whether the model plays as $p_0$ or $p_1$. At each turn, the active player generates $y_t^{(p)}\sim\pi_\theta(\cdot \mid o_t, p)$, where $y_t^{(p)} = \langle\texttt{think}\rangle c_t \langle/\texttt{think}\rangle \langle\texttt{act}\rangle a_t^{(p)} \langle/\texttt{act}\rangle$ follows a reasoning-then-acting paradigm~\citep{yao2023react}, which is strictly enforced via a format reward to guarantee validity, where $c_t$ denotes the reasoning process and $a_t^{(p)}$ is the executable action. This shared-parameter approach creates an automatic curriculum: as the policy improves, it simultaneously faces a stronger opponent.

\paragraph{Role-conditioned Advantage Estimation.}
The vanilla policy gradient suffers from high variance, particularly problematic in self-play where the opponent continuously evolves. Moreover, many games exhibit inherent role asymmetries (e.g., first-mover advantages), causing different expected returns for each role even under optimal play. A global baseline poorly estimates both roles' expectations, leading to biased gradients. To address this, Role-conditioned Advantage Estimation (RAE)~\citep{liu2025spiral,yuan2025mars} maintains separate baselines $b_p$ for each role, estimating the role-specific expected return $\mathbb{E}[R_p(\tau)]$. The advantage is computed as $A_p(\tau) = R_p(\tau) - b_p$, yielding the variance-reduced policy gradient:
\begin{equation}
  \begin{aligned}
      \nabla_\theta J(\theta) = \mathbb{E}_{\tau \sim \pi_\theta} \Big[ \sum_{p \in \{0,1\}} \sum_{t \in T_p} A_p(\tau) \\
      \cdot \nabla_\theta \log \pi_\theta(y_t^{(p)} \mid o_t, p) \Big]
  \end{aligned}
\end{equation}
where $T_p$ denotes turns where player $p$ acted. By centering returns around role-specific expectations, RAE ensures gradient updates reflect genuine strategic improvements rather than inherent positional advantages.

\section{Empirical Investigation of the Impasse}
\label{sec:method_issue}
Self-play RLVR aims to foster continuous improvement through iterative competition. In principle, this paradigm relies on dynamic evolutionary interactions to maintain the diversity of match outcomes, ensuring the variance required for effective policy optimization. However, in open-ended social games, we observe a critical failure mode where agents demonstrate increasing mastery of game mechanics yet stagnate in strategic intelligence. As illustrated in \Cref{fig:method_win}, training on Qwen3-4B-Base yields improvements in basic metrics such as game length and reward. However, the win rate against a fixed Gemini-2 opponent exhibits a declining trend, ultimately converging to a sub-optimal state and failing to demonstrate sustained performance gains. To diagnose the internal dynamics driving this impasse, we introduce \textit{\textbf{Match Entropy}} ($H_{\text{match}}^{(t)}$), which quantifies the diversity of match outcomes at training step $t$ over the distribution $P^{(t)}=\{p_{\text{win}}, p_{\text{draw}}, p_{\text{loss}}\}$:
\begin{equation}
    H_{\text{match}}^{(t)} =- \frac{1}{\log(|P|)}\sum_{o \in P} p_o \log (p_o)
\end{equation}
\begin{figure}[!t]
\centerline{\includegraphics[width=\columnwidth]{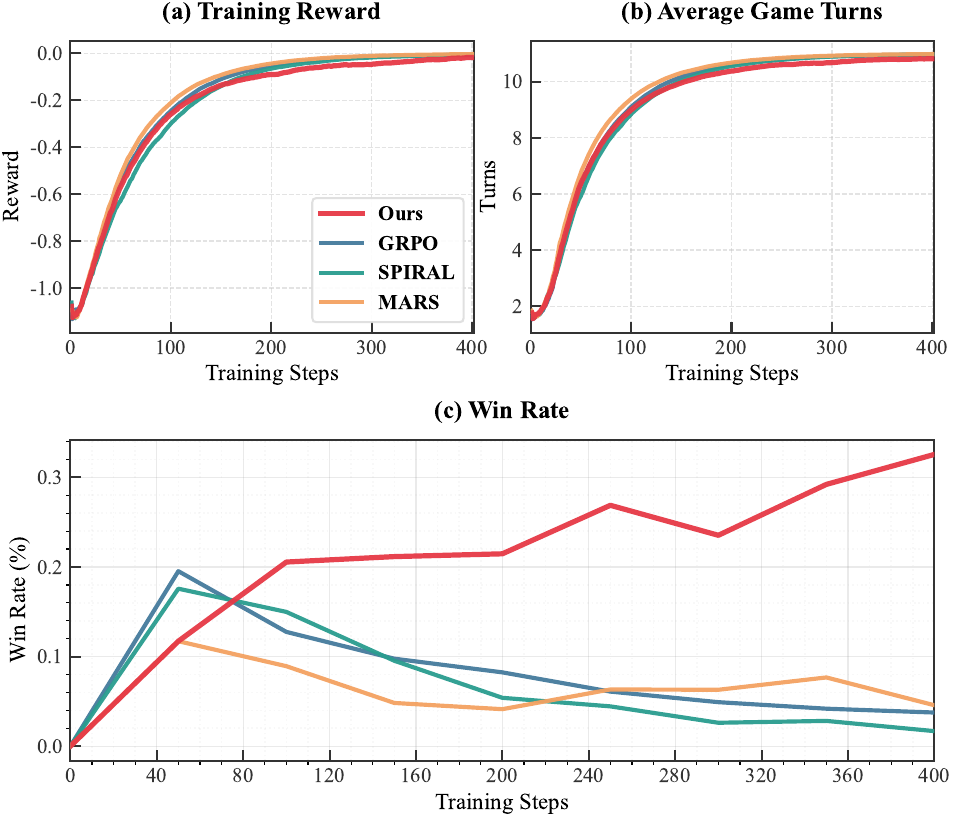}}
\caption{Training dynamics of the self-play RLVR in social language game (Negotiation). While (a) training Reward and (b) average Game Length exhibit normal growth patterns, indicating the acquisition of basic game mechanics, (c) for baselines, the Win Rate against a fixed Gemini-2.0 opponent suffers a significant collapse.}
\label{fig:method_win}
\end{figure}
\Cref{fig:intro}(a) reveals a distinct collapse pattern: after a brief initial exploration phase,$H_{\text{match}}^{(t)}$ rapidly diminishes and remains near zero. This signals convergence to a \textbf{evolution impasse}, a stagnant state where agents become trapped in suboptimal solutions due to the vast, unstructured strategy space of social language games that makes it extremely difficult to sufficiently explore effective strategies (e.g., nuanced persuasion or deception). 

Mathematically, this degeneration directly precipitates the vanishing of gradient signals. In standard policy gradient methods, the update relies on the advantage function $A_p(\tau) = R_p(\tau) - b_p$, where the value baseline $b_p$ estimates the expected return $\mathbb{E}[R_p]$. When the match outcome distribution becomes deterministic (i.e., $H_{\text{match}}^{(t)} \to 0$), the value baseline converges to constant return ($b_p \approx R_t$), causing the advantage to vanish: $A_t \to 0$. This diminishes gradients, preventing escape from local optima. To unlock the reasoning potential of LLMs in social language games, it is paramount to establish a mechanism that can \textit{perceive} this collapse and \textit{intervene} to restore the gradient landscape.

\begin{figure}[!t]
\centerline{\includegraphics[width=\columnwidth]{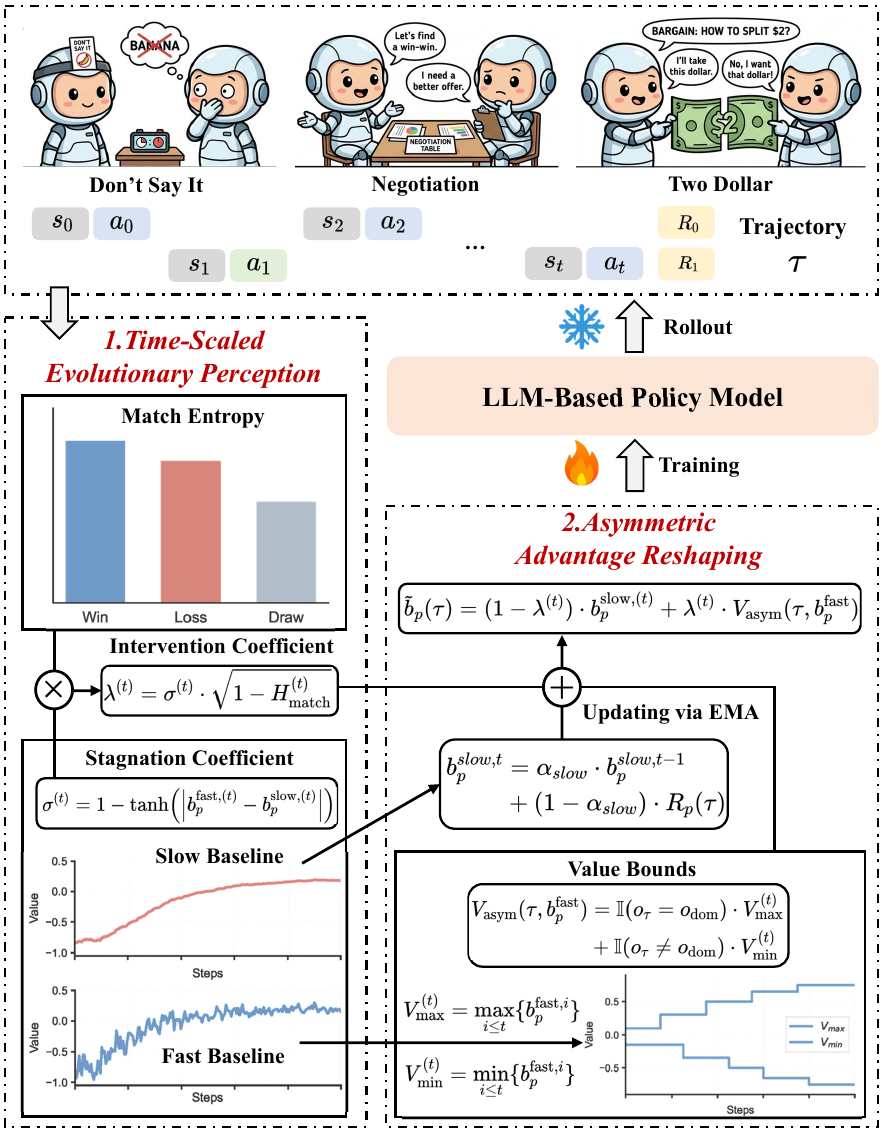}}
\caption{Overview of DEPT. The framework comprises two components: (1) \textbf{\textit{Time-Scaled Evolutionary Perception}}, which maintains fast and slow EMA baselines, combined with Match Entropy to derive the Intervention Coefficient $\lambda^{(t)}$ quantifying evolutionary impasse; and (2) \textbf{\textit{Asymmetric Advantage Reshaping}}, which leverages historical value bounds to suppress dominant outcomes and amplify rare trajectories, restoring gradient signals for continuous policy evolution }
\label{fig:main}
\end{figure}
\section{\fullmethod}
\label{sec:method}
To overcome the evolutionary
impasse, we propose \fullmethod~(\method), an algorithm designed to dynamically perceive evolutionary states and adaptively modulate optimization landscapes to enforce continuous evolution. As shown in \Cref{fig:main}, our algorithm leverages a dual-timescale value baseline to quantify the stationarity of the training process and dynamically injects intrinsic motivation signals into the advantage estimation to enforce continuous evolution.

\subsection{Time-Scaled Evolutionary Perception}
Standard value baselines are inherently reactive: they simply track the current expected return, rendering them incapable of distinguishing between a \textit{stable optimal policy} and a \textit{evolution impasse}. In both cases, the baseline stabilizes, failing to signal that the learning process has stalled. To distinguish these states, we introduce a dual-timescale mechanism that captures the \textit{dynamics} of the training process by contrasting short-term and long-term value estimates. For each player role $p \in \{0, 1\}$, we maintain two parallel value baselines, $b_p^{\text{fast}}$ and $b_p^{\text{slow}}$, updated via Exponential Moving Average (EMA) with distinct decay rates $\alpha_{\text{fast}} < \alpha_{\text{slow}}$:
\begin{equation}
    b_{p}^{k,t} = \alpha_{k} \cdot b_{p}^{k,t-1} + (1 - \alpha_{k}) \cdot R_p(\tau), k \in \{\text{fast}, \text{slow}\}
\end{equation}
Here, $b_p^{\text{fast}}$ rapidly responds to recent outcome returns, while $b_p^{\text{slow}}$ serves as a stable anchor by incorporating a longer history of game outcomes. The divergence between these two signals effectively functions as a differential indicator of training progress: during phases of rapid policy improvement, the fast baseline quickly shifts to match new reward levels, whereas the slow baseline lags behind, creating a significant delta ($|b_p^{\text{fast}} - b_p^{\text{slow}}|$). Conversely, when the policy stabilizes—implying the reward distribution has become stationary—the fast baseline ceases to deviate from the slow baseline, and the two estimates asymptotically converge. We leverage this property to define the Stagnation Coefficient $\sigma^{(t)} \in [0, 1]$ to quantify the stationarity of the training process:
\begin{equation}
    \sigma^{(t)} = 1 - \tanh\left(\left|b_p^{\text{fast},(t)} - b_p^{\text{slow},(t)}\right| \right)
\end{equation}
When the policy evolves actively, the baseline divergence yields low $\sigma^{(t)}$. As the system settles into stagnation, the diminishing divergence drives $\sigma^{(t)}$ higher.

\paragraph{Intervention Coefficient} High stationarity alone does not necessitate intervention, as it characterizes both successful convergence and failure. To specifically isolate degenerate cases, we incorporate the Stagnation Coefficient $\sigma^{(t)}$ with match entropy $H_{\text{match}}^{(t)}$ defined in \Cref{sec:method_issue}. This yields the Intervention Coefficient $\lambda^{(t)} \in [0, 1]$, which determines the necessity of active guidance:
\begin{equation}
\lambda^{(t)}= \sigma^{(t)} \cdot \sqrt{1 - H_{\text{match}}^{(t)}}
\end{equation}

This soft-gating mechanism ensures that intervention strength scales smoothly with the evolutionary state: (1) Healthy Exploration (high $H_{\text{match}}^{(t)}$): Outcome diversity suppresses $\lambda^{(t)}$, minimizing interference during normal learning. (2) Evolution Impasse (low $H_{\text{match}}^{(t)}$, high $\sigma^{(t)}$): The combination of outcome determinism and high stationarity elevates $\lambda^{(t)}$, progressively strengthening the corrective signal.

\begin{table*}[t]
\caption{Main experimental results on three adversarial games. The best results are \textbf{bolded}.}
\small
\centering
\renewcommand{\arraystretch}{1.2}
\resizebox{\linewidth}{!}{
\begin{tabular}{lcccccccccccc}
\toprule
\multirow{2}{*}{Methods}
& \multicolumn{4}{c}{Don't Say It}
& \multicolumn{4}{c}{Negotiation}
& \multicolumn{4}{c}{Two Dollar} \\
\cmidrule(lr){2-5} \cmidrule(lr){6-9} \cmidrule(lr){10-13}
& GPT-5.1 & Gemini-2.5 &Gemini-2.0 & AVG.
& GPT-5.1 & Gemini-2.5 &Gemini-2.0 & AVG.
& GPT-5.1 & Gemini-2.5 &Gemini-2.0 & AVG. \\
\midrule
\multicolumn{13}{l}{\textit{Qwen3-4B-Base}} \\
VANILLA
& 6.25 & 2.47 & 1.43 & 3.39 
& 0.52 & 2.21 & 0.39 & 1.04  
& 1.43 & 0.78 & 2.47 & 1.56 \\
SPAG 
& 37.11 & 26.82 & 14.58 & 26.17 
& 8.98 & 22.95 & 18.35 & 16.76 
& 25.91 & 24.48 & 26.17 & 25.52 \\

GRPO 
& 42.19 & 41.93 & 41.93 & 42.01 
& 15.23 & 22.01 & 20.39 & 19.21 
& 26.56 & 29.43 & 27.73 & 27.91\\

MARS 
& 52.47 & 42.19 & 27.99 & 40.89
& 14.06 & 24.35 & 23.79 & 20.73
& 26.17 & 30.08 & 26.82 & 27.69\\

SPIRAL 
& 48.31 & 42.06 & 47.27 & 45.88  
& 15.10 & 20.57 & 14.83 & 16.84
& 26.30 & 30.08 & 26.56 & 27.65\\

\rowcolor{rowblue}
\method~(Ours) 
& \textbf{56.64} & \textbf{54.95} & \textbf{58.59} & \textbf{56.73} 
& \textbf{17.84} & \textbf{36.46} & \textbf{43.67} & \textbf{32.35}
& \textbf{33.20} & \textbf{36.59} & \textbf{32.42} & \textbf{34.07}\\
\midrule
\multicolumn{13}{l}{\textit{Qwen3-8B-Base}} \\
VANILLA
& 23.24 & 16.41 & 5.86 & 15.17 
& 6.25 & 9.11 & 1.69 & 5.69 
& 1.69 & 1.17 & 3.52 & 2.13\\

SPAG 
& 54.10 & 48.44 & 26.41 & 41.21 
& 9.11 & 19.79 & 11.01 & 13.31 
& 17.58 & 16.93 & 23.44 & 19.31\\

GRPO 
& 33.46 & 40.22 & 30.64 & 34.77  
& 5.47 & 17.97 & 20.68 & 14.70  
& 25.52 & 26.69 & 25.78 & 26.00\\

MARS 
& 51.95 & 41.54 & 28.39 & 40.62  
& 8.46 & 16.80 & 23.04 & 16.10
& 35.16 & 25.13 & 26.82 & 29.04\\

SPIRAL 
& 41.15 & 29.69 & 42.84 & 37.89 
& 8.59 & 22.27 & 21.04 & 17.30
& 25.65 & 29.17 & 23.83 & 26.22\\

\rowcolor{rowblue}
\method~(Ours) 
& \textbf{63.02} & \textbf{52.73} & \textbf{47.92} & \textbf{54.56} 
& \textbf{14.06} & \textbf{32.81} & \textbf{48.76} & \textbf{31.88} 
& \textbf{40.62} & \textbf{38.02} & \textbf{30.86} & \textbf{36.50}\\
\bottomrule
\end{tabular}
}
\label{tab:main_results}
\end{table*}
\subsection{Asymmetric Advantage Reshaping}
\label{sec:reshaping}
When $\lambda^{(t)} \to 1$, the standard advantage signal vanishes $A(\tau) \approx 0$. To restore effective learning, we introduce an Asymmetric Advantage Reshaping, reshaping the advantage function by constructing a contrastive baseline that differentially treats dominant versus rare outcomes. 

\paragraph{Global Performance Bounds} 
Constructing such a contrastive signal requires a rigorous reference scale. Arbitrary baseline modifications can lead to unbounded advantage values, risking gradient instability and distorting the reward scale. To ensure effective reshaping, we track the historical upper and lower bounds of the expected returns. We utilize the fast baseline $b_p^{\text{fast}}$: it is sufficiently sensitive to register performance extrema (potential capabilities) yet robust enough to filter out single-sample noise. We define the historical bounds as:
\begin{equation}
    V_{\max}^{(t)} = \max_{i \le t} \{b_p^{\text{fast},i}\}, \quad V_{\min}^{(t)} = \min_{i \le t} \{b_p^{\text{fast},i}\}
\end{equation}

\paragraph{Asymmetric Value Construction} 
Match entropy collapse stems from one outcome class dominating the batch. To break this homogeneity, we assign different value based on whether a trajectory $\tau$ produces the dominant outcome $o_{\text{dom}}$ or not. We then construct an asymmetric target value $V_{\text{asym}}(\tau, b_p^{\text{fast}})$ that applies differential pressure based on the outcome type. Formally, this asymmetric value is defined as:
\begin{equation}
    \begin{split}
        V_{\text{asym}}(\tau, b_p^{\text{fast}}) &= \mathbb{I}(o_\tau = o_{\text{dom}}) \cdot V_{\max}^{(t)} \\
        &+ \mathbb{I}(o_\tau \neq o_{\text{dom}}) \cdot V_{\min}^{(t)}
    \end{split}
    \end{equation}
This construction induces asymmetric gradient flow through the advantage function $A_\tau = R_\tau - V_{\text{asym}}^{(t)}(\tau)$. For dominant-outcome trajectories, the high $V_{\max}$ reduces positive gradients and amplifies negative gradients, suppressing the over-represented behavior. For rare-outcome trajectories, the low $V_{\min}$ amplifies positive gradients and reduces negative gradients, encouraging exploration of under-represented strategies.

\paragraph{Adaptive Baseline Fusion} Finally, to ensure smooth transitions between normal training and intervention phases, we synthesize the final baseline $\tilde{b}_p(\tau)$ by interpolating between the stable slow baseline and the asymmetric value. The intervention coefficient $\lambda^{(t)}$ governs this fusion:
\begin{equation}
    \tilde{b}_p(\tau) = (1 - \lambda^{(t)}) \cdot b_p^{\text{slow},(t)} + \lambda^{(t)} \cdot V_{\text{asym}}(\tau,b_p^{\text{fast}})
\end{equation}
When $\lambda^{(t)}$ is small, the baseline approximates the standard slow estimate, preserving normal gradient dynamics. As $\lambda^{(t)}$ increases, the asymmetric term gradually dominates, progressively steering the policy away from the collapsed state. The policy is updated using the reshaped advantage $\tilde{A}_p(\tau) = R_p(\tau) - \tilde{b}_p(\tau)$, yielding the final gradient objective:
\begin{equation}
    \begin{aligned}
        \nabla_\theta J(\theta) &= \mathbb{E}_{\tau \sim \pi_\theta} \Big[ \sum_{p \in \{0,1\}} \sum_{t \in T_p} \left( R_p(\tau) - \tilde{b}_p(\tau) \right) \\
       &\quad \cdot \nabla_\theta \log \pi_\theta(y_t^{(p)} \mid o_t, p) \Big]
    \end{aligned}
\end{equation}

The full optimization procedure is shown in \Cref{algorithm:method}. We also provide theoretical analysis about the mathematical foundations of our dual-timescale metrics and the gradient dynamics guaranteeing impasse recovery in \Cref{sec:theoretical_analysis}.

\section{Experiments}
\label{sec:exp}
\subsection{Experimental Settings}
\paragraph{Tasks} We evaluate our method on three adversarial open-ended language tasks involving multi-turn conversation and decision-making: (1) \textbf{Don't Say It}: Players aim to trick their opponent into saying a secret word while avoiding saying their own. The first to speak the word loses; otherwise, it is a draw upon reaching the turn limit. (2) \textbf{Negotiation}: Players negotiate to exchange five resource types based on private valuations, aiming to maximize inventory value. The player with the higher value wins, with ties declared for equal payoffs or reaching turn limit. (3) \textbf{Two Dollar}: Players bargain to split a fixed fund based on hidden role constraints. The player with the higher payoff wins; equal payoffs or reaching the turn limit result in a draw. All three games are adopted from TextArena~\citep{guertler2025textarena}. For further details, please refer to \Cref{app:tasks}.

\paragraph{Baselines} We evaluate our method against several self-play RLVR methods in zero-sum games: (1) \textbf{SPIRAL}~\citep{liu2025spiral}: The first fully online multi-turn self-play RL system for LLMs. It introduces Role-conditioned Advantage Estimation to prevent thinking collapse in zero-sum games; (2) \textbf{MARS}~\citep{yuan2025mars}: It modifies the multi-turn GRPO algorithm by introducing a turn-level advantage estimator and role-specific advantage normalization to stabilize self-play RL training; (3) \textbf{GRPO}: An adaptation of multi-turn GRPO~\citep{shao2024deepseekmath} for self-play settings, where group advantages are calculated separately for each role to preserve effective role-specific learning signals; (4) \textbf{SPAG}~\citep{cheng2024self}: An offline RL method that designs a discounted reward scheme for adversarial language games. Implementation details for all baselines are provided in \Cref{app:baselines}.
\begin{figure}[!t]
    \centerline{\includegraphics[width=\columnwidth]{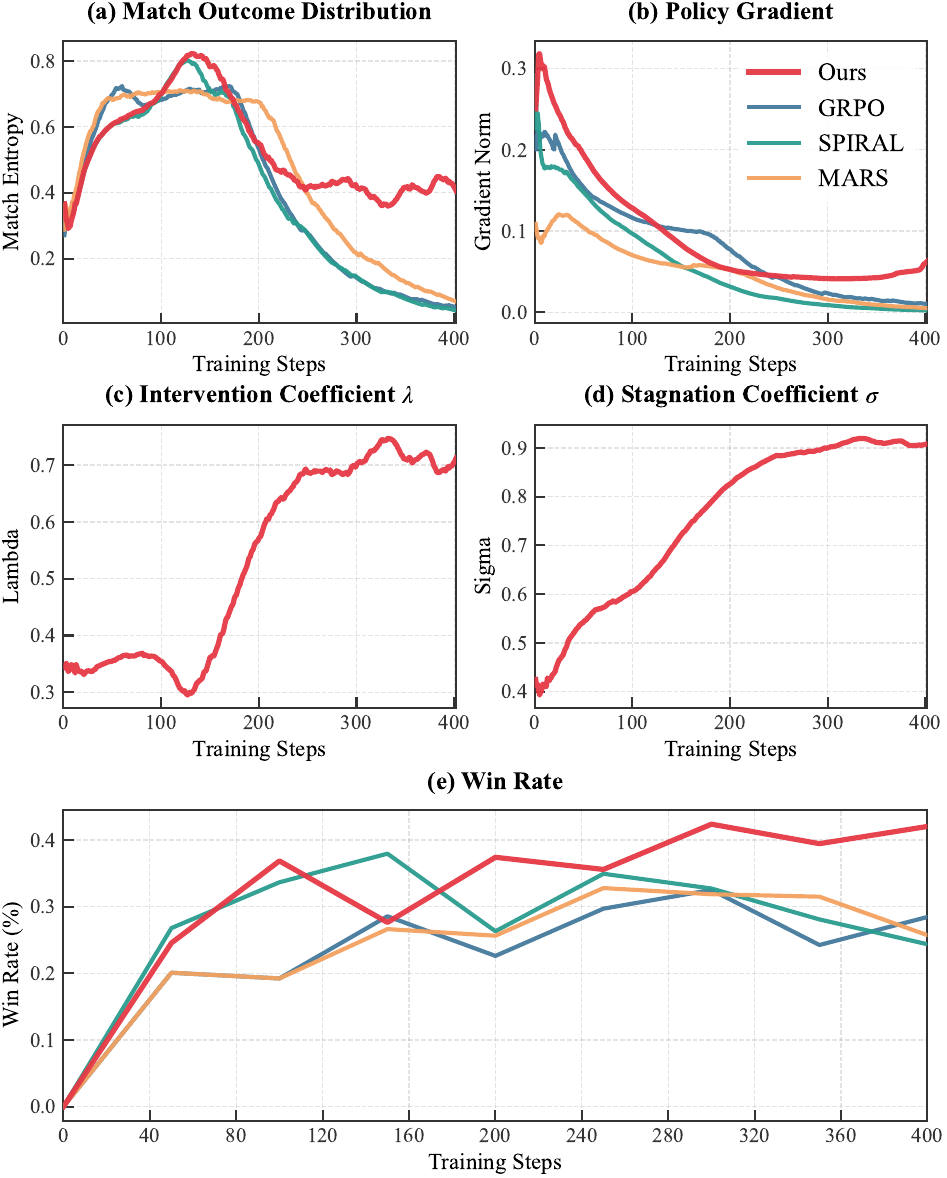}}
    \caption{Training dynamics on Two Dollars. For Don't Say it and Negotiation, please refer to \Cref{fig:exp_dontsayit,fig:exp_negotiation}.}
    \label{fig:exp_twodollar}
    \end{figure}
\paragraph{Evaluation} Following SPIRAL~\citep{liu2025spiral}, we select Gemini-2 (gemini-2.0-flash-lite) as one of our evaluation opponents. To provide a more comprehensive assessment, we additionally evaluate our method against different LLMs: GPT-5.1 opponent (gpt-5.1-2025-11-13) and Gemini-2.5 opponent (gemini-2.5-flash-lite). Performance is quantified by the win rate against these opponents, defined as $\frac{1}{N} \sum_{i=1}^{N} \mathbb{I}[R_i = 1]$, where $R_i$ denotes the final reward for the $i$-th trajectory (with $R_i=1$ indicating a win). For each task, we conduct experiments using three random seeds (42, 100, 200), evaluating $N = 256$ episodes per seed. The evaluation is symmetric, granting an equal opportunities to act as the first and second player.

\paragraph{Training} We employ Qwen3-4B-Base and Qwen3-8B-Base~\citep{yang2025qwen3} as the backbone models for all experiments. Following SPIRAL~\citep{liu2025spiral}, we set the training steps to 400 with 128 samples collected per step, resulting in a total of 51,200 game transitions for each task. The training batch size is set to 128. For the dual-baseline mechanism, we use $\alpha_{\text{fast}}=0.5$ and $\alpha_{\text{slow}}=0.95$. Detailed RL training hyperparameters are provided in \Cref{app:hyperparameters}.

\subsection{Experimental Results}
\paragraph{RQ1: Is \method~effective for social language agents?}
As shown in \Cref{tab:main_results}, \method~achieves the highest average win rate across all three games, consistently outperforming all baselines on both backbones. On Qwen3-4B-Base, \method~attains 32.35\% in Negotiation, surpassing the strongest baseline (MARS: 20.73\%) by a relative margin of 56\%. Similar gains are observed in Don't Say It (+10.8\% over SPIRAL) and Two Dollar (+6.16\% over GRPO). These improvements scale consistently to Qwen3-8B-Base, where \method~nearly doubles the Negotiation win rate compared to SPIRAL (31.88\% vs. 17.30\%). The limited performance of competitive baselines aligns with the evolution impasse described in \Cref{sec:method_issue}. In contrast, \method's consistent superiority against diverse opponents demonstrates that our dual-timescale perception and asymmetric reshaping mechanism effectively prevents strategic collapse, enabling robust social reasoning.

\paragraph{RQ2: Can \method~effectively perceive and intervene during training?}
As shown in \Cref{fig:exp_twodollar,fig:exp_dontsayit,fig:exp_negotiation}, our method effectively mitigates match entropy collapse and preserves informative gradient signals, enabling the discovery of better strategies as reflected by steadily increasing win rates. In contrast, the baseline converges to evolutionary
impasse with suboptimal policies. Furthermore, the proposed Intervention Coefficient accurately detects stagnation: during the initial exploration phase, the intervention remains low; as training progresses toward convergence, the Stagnation Coefficient increases, becoming more responsive to changes in match entropy and adaptively adjusting the intervention strength to maintain sufficient match entropy.
\begin{figure}[!t]
    \centerline{\includegraphics[width=\columnwidth]{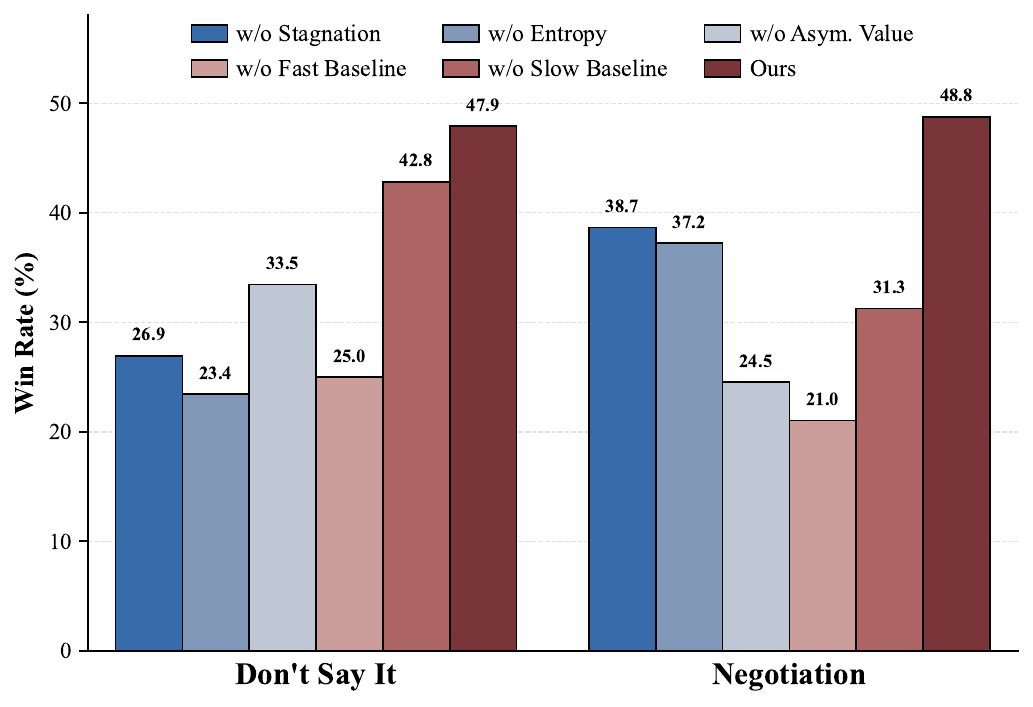}}
    \caption{Ablation study for \method.}
    \label{fig:abla_win}
\end{figure}

\subsection{Analysis}
\paragraph{RQ3: How does \method~work?} To examine the effectiveness of each component in \method, we conduct ablation studies using Qwen3-8B-Base, as reported in \Cref{fig:abla_win}. All ablations lead to performance degradation, confirming the necessity of each design choice. \textit{\textbf{(1) Effect of perception.}} Removing the Stagnation Coefficient causes aggressive entropy correction during non-stationary training phases, leading to substantial advantage estimation errors that disrupt policy exploration. Removing match entropy gating results in continuous penalties even when the outcome distribution maintains sufficient diversity, forcing ineffective random exploration and preventing convergence. \textit{\textbf{(2) Effect of asymmetric value.}} Without asymmetric target values, the method cannot selectively suppress over-represented behaviors while supporting under-represented ones, leading to convergence at suboptimal low-entropy states. \textit{\textbf{(3) Effect of dual-baseline.}} Using either the fast or slow baseline alone degrades performance. This confirms that the effectiveness of \method~arises from the synergy between dual baselines, which jointly enable accurate perception and targeted intervention.
\begin{table}[t]
    \caption{Out-of-distribution evaluation on increased complexity games.}
    \small
    \centering
    \renewcommand{\arraystretch}{1.2}
    \resizebox{\linewidth}{!}{
    \begin{tabular}{lccc}
    \toprule
    Methods  
    & \makecell{HardCore\\Don'tSayIt}
    & \makecell{RandomValue\\Negotiation}
    & \makecell{VariableSum\\Dollar} \\
    \midrule
    \multicolumn{4}{l}{\textit{Qwen3-4B-Base}} \\
    VANILLA
    & 0.39 & 0.13 & 3.78 \\
    SPAG 
    & 7.03 & 11.85 & 26.17\\
    
    GRPO 
    & 22.01 & 14.52 & 27.99\\
    
    MARS 
    & 12.11 & 12.55 & 25.39\\
    
    SPIRAL 
    & 17.97 & 12.56 & 23.96\\
    \rowcolor{rowblue}
    \method~(Ours) 
    & \textbf{22.92} & \textbf{15.36} & \textbf{30.47}\\
    \midrule
    \multicolumn{4}{l}{\textit{Qwen3-8B-Base}} \\
    VANILLA
    & 2.34 & 3.78 & 2.73 \\
    
    SPAG 
    & 12.12 & 13.02 & 26.56\\
    
    GRPO 
    & 10.03 & 6.64 & 19.14\\
    
    MARS 
    & 8.33 & 7.56 & 24.48\\
    
    SPIRAL 
    & 18.88 & 6.52 & 23.18\\
    
    \rowcolor{rowblue}
    \method~(Ours) 
    & \textbf{19.27} & \textbf{17.08} & \textbf{27.73}\\
    \bottomrule
    \end{tabular}
    }
    \label{tab:ood_game_results}
\end{table}
\begin{table*}[t]
    \caption{Out-of-distribution on reasoning benchmarks. The best results are \textbf{bolded}. The second results are \underline{underline}.}
    \small
    \centering
    \renewcommand{\arraystretch}{1.2}
    \resizebox{\linewidth}{!}{
    \begin{tabular}{lccccccccc}
    \toprule
    Methods
    & Minerva & AMC-23 & AIME24 & AIME25 & Math500 & Olympiad
    & GPQA-D & MMLU-Pro & Average \\
    \midrule
    \multicolumn{10}{l}{\textit{Qwen3-4B-Base}} \\
    VANILLA
    &22.36 &42.03 &9.58 &6.88 &65.80 &34.52 
    &28.79 &39.36
    &31.07 \\
    SPAG 
    &24.70 & 44.27 & \underline{10.69} &7.78 &\underline{70.87}  &\underline{36.15}  
    &33.33 &52.71    
    &35.06 \\
    GRPO
    &23.25 &38.65 &7.43  &5.80  &56.26 
    &26.64 &\underline{35.86} &49.14 
    &30.38 \\
    MARS
    &\underline{29.57} &45.47 &9.41  &\underline{8.47}  &69.25 
    &35.18 &34.01 &\underline{53.34} 
    &\underline{35.59} \\
    SPIRAL
    &27.17 &\underline{45.73} &9.55  &7.81  &67.95 
    &34.15 &34.18 &51.32 
    &34.73 \\
    \rowcolor{rowblue}
    \method~(Ours)
    &\textbf{32.57}&\textbf{48.70}&\textbf{11.22}&\textbf{10.03}&\textbf{74.64}
    &\textbf{38.79}&\textbf{37.04}&\textbf{56.45}
    &\textbf{38.68} \\
    \midrule
    \multicolumn{10}{l}{\textit{Qwen3-8B-Base}} \\
    VANILLA
    &24.15 &47.66 &13.44 &10.00 &67.54 
    &35.80 &29.80 &46.97 
    &34.42 \\
    SPAG 
    &26.28 &50.42 &12.78 &10.76 &69.67 
    &37.21 &31.14 &51.46 
    &36.22 \\
    GRPO
    &\underline{34.24} &\underline{53.75} &12.50 &\underline{12.05} &\textbf{75.94}
    &\underline{40.56} &35.69 &\textbf{59.05}
    &\underline{40.47} \\
    MARS
    &33.75 &53.59 &\textbf{14.03}&11.74 &74.63 
    &40.35 &\underline{37.37} &57.80 
    &40.41 \\
    SPIRAL
    &31.71 &53.59 &\underline{13.54} &10.38 &72.81 
    &38.40 &35.69 &\underline{58.28} 
    &39.30 \\
    \rowcolor{rowblue}
    \method~(Ours)
    &\textbf{35.00}&\textbf{57.03}&13.06 &\textbf{12.43}&\underline{74.98} 
    &\textbf{40.83}&\textbf{38.72}&57.60 
    &\textbf{41.21} \\
    \bottomrule
    \end{tabular}
    }
    \label{tab:ood_reasoning_results}
\end{table*}

\paragraph{RQ4: How does \method~perform on out-of-distribution evaluation?} To further assess generalization, we evaluate \method~on out-of-distribution games, QA, and mathematical reasoning tasks to examine whether it effectively enhances strategic capabilities. For out-of-distribution games, we evaluate on three more complex, held-out variants—HardCore Don'tSayIt, RandomValue Negotiation, and VariableSum Dollar—to test generalization from their simpler counterparts (e.g., from Negotiation to RandomValue Negotiation). We still report the average win rate across 3 seeds (42, 100, 200), 256 episodes per seed. For QA and mathematical reasoning, we adopt eight benchmarks: MATH500, OlympiadBench, Minerva Math, AIME24, AIME25, and AMC23 for mathematics; GPQA-Diamond and MMLU-Pro for QA. Evaluation details are deferred to the \Cref{app:ood_game,app:ood_math_qa}.

As shown in \Cref{tab:ood_game_results,tab:ood_reasoning_results}, \method~demonstrates superior out-of-distribution generalization. In out-of-distribution games, our approach maintains robust performance even in more complex environments. For QA and mathematics, self-play on games enables effective reasoning ability transfer to general reasoning tasks. This aligns with prior findings~\citep{cheng2024self,liu2025spiral,yuan2025mars} that game-based self-play effectively elicits reasoning capabilities—since \method~achieves stronger in-game performance with more strategic behaviors, it consequently yields greater improvements on reasoning benchmarks. The improvements are particularly pronounced on Qwen3-4B-Base, where consistent gains are observed across all benchmarks.

\paragraph{RQ5: Additional results.} We provide several supplementary analyses in the appendix to further validate the effectiveness of \method. First, we examine the semantic diversity of generated actions to verify whether \method~can produce more varied and contextually appropriate behaviors rather than repetitive patterns. The results in \Cref{fig:diverse} demonstrate that \method~achieves higher semantic richness compared to baselines, indicating that our dual-process mechanism encourages more creative and diverse action generation (See \Cref{app:diverse}). Second, we conduct a sensitivity analysis on the key hyperparameter $\alpha_{\text{fast}}$, which controls the balance between fast and slow thinking modes. Our analysis in \Cref{tab:exp_hyper} reveals that \method~maintains stable performance across the range [0.4, 0.6], suggesting that our method is robust to hyperparameter choices and does not require extensive tuning for deployment (See \Cref{app:hyper}). Third, we analyze the computational overhead introduced by our dual-process architecture to assess its practical applicability. The results in \Cref{fig:buget} show that \method~introduces negligible additional cost (less than 0.0016\% of the total per-iteration training time)  while achieving substantial performance improvements, making it suitable for real-world applications where efficiency is critical (See \Cref{app:computation}). Additionally, we perform comprehensive statistical significance tests to rigorously validate our experimental findings. The results in \Cref{tab:static} confirm that \method~provides statistically significant gains over all baselines (p < 0.05), demonstrating that the observed improvements are not due to random variance (See \Cref{app:static}).

\section{Related Work}
\noindent \textbf{LLM-based Reinforcement Learning} Reinforcement learning for LLMs has evolved from Reinforcement Learning from Human Feedback (RLHF) to RLVR~\citep{NEURIPS2022_b1efde53,guo2025deepseek,team2025kimi,yue2025does,liu2025understanding,tan2025bottom,nie2026attnpo,tan2026pyxpy}, leveraging objective ground truth to unlock deep potential in mathematical reasoning and code generation. However, existing RLVR research focuses on closed-ended tasks with well-defined answers. Real-world open-ended social interactions---such as negotiation and deception---are fundamentally different~\citep{meta2022human,park2023generative,zhou2024sotopia,xu2024language,wang-etal-2025-demo,zhang2026expseek}: they involve infinite action spaces, lack unique correct strategies, and require real-time reasoning about dynamic opponent intentions~\citep{abdulhai2025lmrl,yang2025aria,wang2025adaptive}. Our work addresses the challenge of applying RLVR to open-ended social interaction while sustaining continuous evolution in dynamic environments.

\noindent \textbf{LLM Self-play Training on Language Game}
As a classical paradigm, Self-play~\citep{silver2018general,bansal2018emergent,chen2024self,cheng2024self} addresses data scarcity while constructing an automatic curriculum through opponent co-evolution. For language games in LLMs, SPIRAL~\citep{liu2025spiral} and MARS~\citep{yuan2025mars} introduce RAE to handle role asymmetry, achieving strong performance in limited action spaces. However, in more complex open-ended language domains, these methods suffer from evolutionary
impasse: LLMs converge to homogeneous behaviors, collapsing outcome distributions into deterministic results. This drives advantages toward zero, causing gradient vanishing and trapping models in suboptimal states. Our \method~addresses this through dual-baseline perception to detect stagnation, combined with asymmetric advantage reshaping to restore optimization landscape and ensure continuous policy evolution.

\section{Conclusions}
In this paper, we address the critical challenge of evolution impasse in applying self-play RLVR to open-ended social language games, where agent training is susceptible to gradient vanishing and strategic stagnation. We propose \method, which dynamically perceives evolutionary states and adaptively adjusts the optimization landscape via dual-timescale value baselines. By quantifying stagnation through baseline divergence and match entropy, \method~selectively intervenes to penalize dominant outcomes while amplifying rare trajectories, effectively restoring gradient signals and enforcing continuous strategic exploration. Extensive empirical results confirm that our approach not only surpasses SOTA methods but also maintains superior out-of-distribution generalization capabilities.

\section*{Limitations}

While \method~demonstrates substantial improvements in preventing evolutionary
impasse during self-play training, several limitations offer avenues for future research. 
(1) Self-play reinforcement learning inherently demands computational resources. Our experiments required approximately 30 GPU-hours per run on 8 NVIDIA A800 GPUs. These resource requirements may limit accessibility for researchers with constrained computational budgets.
(2) Following the reasoning-then-acting paradigm, our method generates extended response sequences containing both reasoning traces ($c_t$) and executable actions ($a_t$). This increases inference latency compared to direct action generation, a common trade-off in test-time scaling approaches. The computational cost becomes more pronounced when scaling to longer interaction horizons or larger LLM backbones. 
We believe our dual-scale evolutionary policy training method provides a principled foundation that could inspire novel mechanisms for detecting and intervening in more complex evolutionary dynamics, and we hope this work motivates the community to explore these promising directions.

\section*{Acknowledgments}
This work is supported in part by the National Natural Science Foundation of China under Grants \#72293575, \#72225011 and \#72434005. We also thank Longze Chen and Jiaming Li for their valuable suggestions on the experimental design of this paper.

\bibliography{custom}

\appendix
\newtheorem{definition}{Definition}[section]
\newtheorem{theorem}{Theorem}[section]
\newtheorem{lemma}{Lemma}[section]
\newtheorem{proposition}{Proposition}[section]

\section{Theoretical Analysis}
\label{sec:theoretical_analysis}

In this section, we provide a theoretical analysis of \method. We first establish the mathematical foundation of the Time-scaled Evolutionary Perception mechanism, proving that the proposed metrics serve as estimators for the velocity and variance of the reward distribution. Subsequently, we analyze the gradient dynamics, proving that our method ensures effectively restoring gradient signals in evolutionary
impasse via a synthetic variance injection mechanism.

\subsection{Mathematical Foundation of Evolutionary Perception}
\label{sec:math_perception}

The core of \method~lies in quantifying the evolutionary state using baseline divergence and match entropy. Here, we provide the formal derivations linking these metrics to the underlying reward dynamics.

\subsubsection{Baseline Divergence as Velocity Estimation}
In the main text, we utilize the divergence between fast and slow baselines to measure training stationarity. We now prove that this metric is mathematically equivalent to an estimator of the \textit{velocity of expected return}.

\begin{proposition}[Divergence-Velocity Equivalence]
Let $\mu(t) = \mathbb{E}[R_p(\tau)]$ be the underlying expected return at step $t$. Assuming $\mu(t)$ is locally linear within the effective window of the Exponential Moving Average (EMA), the expected divergence between two baselines with decay rates $\alpha_{\text{slow}} > \alpha_{\text{fast}}$ is proportional to the magnitude of the time derivative of the expected return:
\begin{equation}
    \mathbb{E}\left[ \left| b_p^{\text{fast}, t} - b_p^{\text{slow}, t} \right| \right] \approx C \cdot \left| \frac{d\mu(t)}{dt} \right|,
\end{equation}
where $C > 0$ is a constant determined by the time constants.
\end{proposition}

\begin{proof}
Recall the EMA update rule: $b_p^{k,t} = \alpha_k b_p^{k,t-1} + (1-\alpha_k)R_p(\tau)$. By recursively expanding this update, the baseline at time $t$ can be expressed as a weighted sum of historical rewards. Taking the expectation and performing a first-order Taylor expansion of the expected return $\mu(t-j) \approx \mu(t) - j \cdot \dot{\mu}(t)$:
{\small
\begin{align}
    \mathbb{E}[b_p^{k,t}] &= (1 - \alpha_k) \sum_{j=0}^{\infty} \alpha_k^j \mu(t-j) \nonumber \\
    &\approx (1 - \alpha_k) \sum_{j=0}^{\infty} \alpha_k^j \left( \mu(t) - j \cdot \dot{\mu}(t) \right) \nonumber \\
    &= \mu(t) \underbrace{(1 - \alpha_k) \sum_{j=0}^{\infty} \alpha_k^j}_{=1} - \dot{\mu}(t) (1 - \alpha_k) \sum_{j=0}^{\infty} j \alpha_k^j.
\end{align}
}%
Using the geometric series identity $\sum_{j=0}^{\infty} j x^j = \frac{x}{(1-x)^2}$ for $|x|<1$, we substitute $x=\alpha_k$:
\begin{equation}
    \mathbb{E}[b_p^{k,t}] \approx \mu(t) - \dot{\mu}(t) \cdot \frac{\alpha_k}{1 - \alpha_k}.
\end{equation}
Let $\mathcal{T}_k = \frac{\alpha_k}{1 - \alpha_k}$ denote the effective time lag. The expected divergence is:
\begin{equation}
\begin{split}
    \mathbb{E}[\delta^{(t)}] &= \left| \mathbb{E}[b_p^{\text{fast}, t}] - \mathbb{E}[b_p^{\text{slow}, t}] \right| \\
    &\approx \left| (\mu(t) - \mathcal{T}_{\text{fast}}\dot{\mu}(t)) - (\mu(t) - \mathcal{T}_{\text{slow}}\dot{\mu}(t)) \right| \\
    &= (\mathcal{T}_{\text{slow}} - \mathcal{T}_{\text{fast}}) \cdot |\dot{\mu}(t)|.
\end{split}
\end{equation}
Since $\alpha_{\text{slow}} > \alpha_{\text{fast}}$, we have $\mathcal{T}_{\text{slow}} > \mathcal{T}_{\text{fast}}$, ensuring the coefficient is positive. Thus, the baseline divergence serves as a direct proxy for the training velocity $|\dot{\mu}(t)|$.
\end{proof}

\subsubsection{Match Entropy as Time-Varying Variance Proxy}
We now establish the mathematical relationship between the observed Match Entropy $H_{\text{match}}^{(t)}$ and the underlying reward variance.

\begin{lemma}[Entropy-Variance Correlation]
Let $\nu(t) = \text{Var}[R_p(\tau)]$ be the underlying reward variance at time step $t$. As the match entropy vanishes, the reward variance asymptotically approaches zero:
\begin{equation}
    \lim_{H_{\text{match}}^{(t)} \to 0} \nu(t) = 0.
\end{equation}
Consequently, maintaining high entropy is a necessary condition for maintaining non-zero reward variance.
\end{lemma}
\label{lemma:entropy_variance}

\begin{proof}
Consider the discrete outcome space $\mathcal{O} = \{o_1, \dots, o_K\}$ (e.g., \{Win, Loss, Draw\}) with associated fixed reward values $\{r_1, \dots, r_K\}$. At time $t$, the policy $\pi_{\theta_t}$ induces a probability distribution $p_i(t) = P(R_p(\tau) = r_i)$.

The variance $\nu(t)$ is formally defined as:
\begin{equation}
    \nu(t) = \sum_{i=1}^K p_i(t) \left( r_i - \mathbb{E}[R_p(\tau)] \right)^2.
\end{equation}
From information theory, the entropy $H_{\text{match}}^{(t)} = - \sum p_i(t) \log p_i(t)$ approaches zero if and only if the probability mass concentrates on a single dominant outcome $o_k$. Mathematically, this implies the limit behavior:
\begin{equation}
    p_k(t) \to 1 \quad \text{and} \quad p_{j \ne k}(t) \to 0.
\end{equation}
In this limit, the expected return converges to the reward of the dominant outcome: $\mathbb{E}[R_p(\tau)] \to r_k$. Substituting these limits into the definition of $\nu(t)$:
\begin{equation}
\begin{split}
    \lim_{H_{\text{match}}^{(t)} \to 0} \nu(t) &= \lim_{p_k \to 1} \Big[ p_k(t)(r_k - r_k)^2 \\
    &\quad + \sum_{j \ne k} p_j(t)(r_j - r_k)^2 \Big] \\
    &= 1 \cdot 0 + 0 = 0.
\end{split}
\end{equation}
This derivation proves that $H_{\text{match}}^{(t)}$ and $\nu(t)$ are positively correlated in the asymptotic limit: a collapse in match entropy strictly necessitates a collapse in the reward variance $\nu(t)$.
\end{proof}

\subsubsection{Bayesian Derivation of Intervention Coefficient}
We presented the intervention coefficient $\lambda^{(t)}$ in a specific functional form in the main text. Here, we derive this form from a general Bayesian perspective.

\begin{theorem}[Bayesian Form of $\lambda$]
Let $Z_t \in \{0, 1\}$ be a latent variable where $Z_t=1$ indicates a evolution impasse. The coefficient $\lambda^{(t)}$ represents the posterior probability $P(Z_t=1 \mid \mathcal{O}_t)$ given observations $\mathcal{O}_t$.
\end{theorem}

\begin{proof}
We decompose the degenerate state into two independent necessary conditions: Stationarity ($\mathcal{S}$) and Determinism ($\mathcal{D}$). Using Bayes' theorem and assuming conditional independence:
\begin{equation}
    \lambda^{(t)} \triangleq P(Z_t=1) \propto P(\mathcal{S} \mid \delta^{(t)}) \cdot P(\mathcal{D} \mid H_{\text{match}}^{(t)}).
\end{equation}
To map observations to probabilities, we introduce two kernel functions $\mathcal{K}_{\mathcal{S}}$ and $\mathcal{K}_{\mathcal{D}}$ with range $[0, 1]$:
\begin{enumerate}
    \item \textbf{Stagnation Kernel $\mathcal{K}_{\mathcal{S}}$:} Since $\delta^{(t)} \propto |\text{Velocity}|$, stationarity implies $\delta^{(t)} \to 0$. We select the hyperbolic tangent kernel:
    \begin{equation}
         \sigma^{(t)} = P(\mathcal{S} \mid \delta^{(t)}) = 1 - \tanh(\delta^{(t)}).
    \end{equation}
    \item \textbf{Determinism Kernel $\mathcal{K}_{\mathcal{D}}$:} Since low entropy implies determinism, we select the radical kernel to ensure high sensitivity near the boundary of exploration ($H \approx 1$):
    \begin{equation}
        \gamma^{(t)} = P(\mathcal{D} \mid H_{\text{match}}^{(t)}) = \sqrt{1 - H_{\text{match}}^{(t)}}.
    \end{equation}
\end{enumerate}
Substituting these specific kernels yields the instantiation used in \method: $\lambda^{(t)} = \sigma^{(t)} \cdot \gamma^{(t)}$.
\end{proof}

\subsection{Gradient Dynamics Analysis}
\label{sec:gradient_analysis}
\subsubsection{The Vanishing Gradient of Standard Methods}
We first prove why standard methods fail in evolutionary
impasse.

\begin{theorem}[Variance-Dependent Gradient Bound]
\label{thm:grad_bound}
Assume the policy $\pi_\theta$ is Lipschitz continuous such that the norm of the score function is bounded by a constant $G_{\max}$. The magnitude of the expected policy gradient in standard Advantage Estimation is upper-bounded by the standard deviation of the reward:
\begin{equation}
    \| \nabla_\theta J(\theta) \| \le G_{\max} \cdot \sqrt{\nu(t)}.
\end{equation}
\end{theorem}

\begin{proof}
We start with the standard gradient definition $\nabla_\theta J = \mathbb{E}_{\tau}[(R-\mu)\nabla_\theta \log \pi_\theta]$.
First, we invoke the \textbf{Lipschitz continuity assumption}\cite{papini2018stochastic,agarwal2019reinforcement}: for smooth neural policies, the gradient of the log-likelihood is bounded:
\begin{equation}
    \sup_{\tau} \| \nabla_\theta \log \pi_\theta(\tau) \| \le G_{\max}.
\end{equation}
Next, we apply the \textbf{Cauchy-Schwarz inequality} for random variables ($\mathbb{E}[XY]^2 \le \mathbb{E}[X^2]\mathbb{E}[Y^2]$):
\begin{equation}
\begin{split}
    \| \nabla_\theta J \|^2 &= \left\| \mathbb{E} \left[ (R-\mu) \nabla \log \pi \right] \right\|^2 \\
    &\le \mathbb{E} \left[ (R-\mu)^2 \right] \cdot \mathbb{E} \left[ \| \nabla \log \pi \|^2 \right] \\
    &= \nu(t) \cdot \mathbb{E} \left[ \| \nabla \log \pi \|^2 \right].
\end{split}
\end{equation}
Substituting the bound $G_{\max}$:
\begin{equation}
    \| \nabla_\theta J \|^2 \le \nu(t) \cdot G_{\max}^2 \implies \| \nabla_\theta J \| \le G_{\max} \sqrt{\nu(t)}.
\end{equation}
\textbf{Implication:} In a evolution impasse, Lemma~\ref{lemma:entropy_variance} dictates that $\nu(t) \to 0$. Consequently, the gradient magnitude $\| \nabla J \|$ is forced to zero. The optimization halts specifically because the \textit{signal variance} $\nu(t)$ required to drive the gradient has vanished.
\end{proof}

\subsubsection{Gradient Restoration}
We prove that \method~can effectively recovers gradient. 

\begin{theorem}[Push-Pull Gradient Restoration]
In a degenerate state where natural variance $\nu(t) \to 0$, \method~constructs a non-vanishing synthetic gradient field driven by the historical value range $[V_{\min}, V_{\max}]$.
\end{theorem}

\begin{proof}
Consider a training batch $\mathcal{B}$ of size $B$, decomposed into a set of dominant trajectories $\mathcal{D}_{\text{dom}}$ (size $M$) and rare trajectories $\mathcal{D}_{\text{rare}}$ (size $N$), such that $B = M + N$ and $M \gg N$ (indicating a degenerate state).
When intervention is triggered ($\lambda^{(t)} \to 1$), the gradient update $\nabla J_p$ is calculated as the summation over these two subsets:
\begin{equation}
\begin{split}
    \nabla J_p &\approx \frac{1}{B} \Bigg[ \underbrace{\sum_{\tau \in \mathcal{D}_{\text{dom}}} (R_p(\tau) - V_{\max}) \nabla \log \pi_p(\tau)}_{\text{Aggregate Push Force}} \\
    &\quad + \underbrace{\sum_{\tau' \in \mathcal{D}_{\text{rare}}} (R_p(\tau') - V_{\min}) \nabla \log \pi_p(\tau')}_{\text{Aggregate Pull Force}} \Bigg].
\end{split}
\end{equation}
\begin{enumerate}
    \item \textbf{Symmetric Dilemma:} In standard methods, the baseline $b \approx \frac{1}{M}\sum R_p(\tau_{\text{dom}})$. Thus, the advantage for the $M$ dominant samples approaches zero. Since $N \to 0$, the total gradient vanishes.
    \item \textbf{Aggregate Push Mechanism ($M$ samples):} For the dominant set, we set the baseline to $V_{\max}$. Since $V_{\max} \ge R_p(\tau)$ for $\tau \in \mathcal{D}_{\text{dom}}$, the term $(R_p(\tau) - V_{\max})$ is strictly non-positive. Although individual gradients may be small, the summation over $M$ samples generates a substantial negative gradient field, actively suppressing the probability of the stagnant strategy.
    \item \textbf{Aggregate Pull Mechanism ($N$ samples):} For the rare set, we set the baseline to $V_{\min}$. The term $(R_p(\tau') - V_{\min})$ is maximized, assigning a high positive weight to these sparse signals. This amplifies the exploration direction even when $N$ is small.
\end{enumerate}
Through this summation, \method~effectively injects a \textbf{Synthetic Variance} $\nu_{\text{syn}} \propto (V_{\max} - V_{\min})^2$ into the batch. Unlike the natural variance $\nu(t)$ which vanishes as $N \to 0$, this synthetic variance remains robust, ensuring $\| \nabla J_p \| > 0$ and forcing the policy distribution to shift away from the evolution impasse.
\end{proof}

\begin{figure}[!t]
    \centerline{\includegraphics[width=\columnwidth]{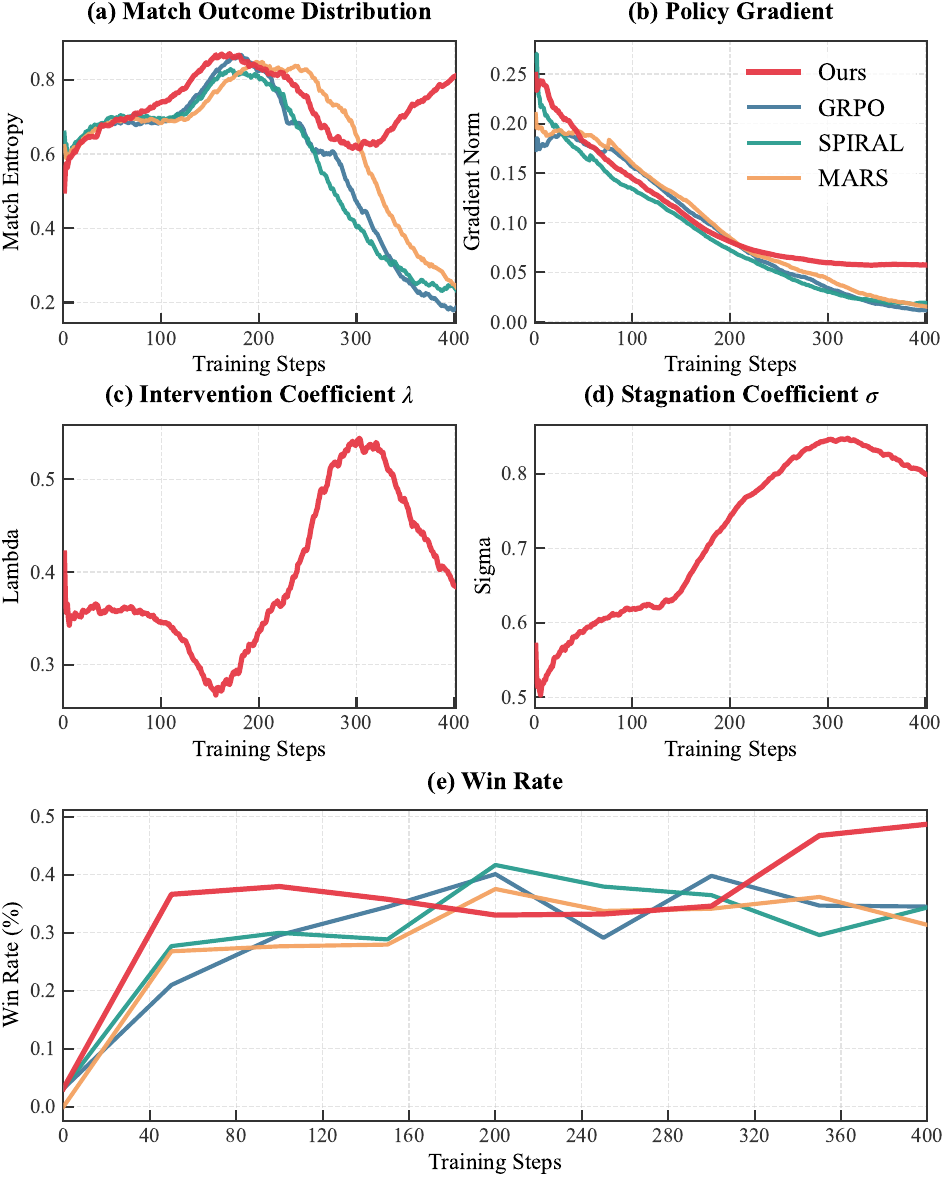}}
    \caption{Training dynamics on Don't Say It.}
    \label{fig:exp_dontsayit}
\end{figure}

\begin{algorithm*}[t]
    \caption{\method: \fullmethod}
    \label{algorithm:method}
    \begin{algorithmic}[1]
    \Require Policy $\pi_\theta$, decay rates $\alpha_{\text{fast}} < \alpha_{\text{slow}}$
    \State Initialize baselines $b_{p}^{\text{fast}} = b_{p}^{\text{slow}} = 0$ for $p \in \{0,1\}$
    \While{not converged}
        \State \textbf{// Self-Play Trajectory Collection}
        \State Collect batch $\mathcal{B} = \{\tau_k\}_{k=1}^{K}$ via self-play, where each $\tau_k$ contains alternating 
        \Statex \quad\quad actions $\{y_t^{(p)}\}$ from $\pi_\theta(\cdot \mid o_t, p)$ and terminal rewards $R_0 = -R_1 = \rho(s_T)$
        
        \State \textbf{// Batch-Level Statistics}
        \State Compute outcome distribution $P(o \mid p)$ for $o \in \{\text{win}, \text{draw}, \text{loss}\}$ from $\mathcal{B}$
        \State Identify dominant outcome $o_{\text{dom}}^{p} \leftarrow \arg\max_{o} P(o \mid p)$ for each $p$
        \State Compute normalized match entropy $H_{\text{match}} \leftarrow -\frac{1}{\log 3} \sum_{o} P(o) \log P(o)$
        
        \State \textbf{// Dual-Timescale Update \& Asymmetric Advantage Reshaping}
        \For{$\tau \in \mathcal{B}$}
            \For{$p \in \{0, 1\}$}
                \State \textit{// Update dual-timescale baselines and global bounds}
                \State $b_{p}^{\text{fast}} \leftarrow \alpha_{\text{fast}} \cdot b_{p}^{\text{fast}} + (1-\alpha_{\text{fast}}) \cdot R_p(\tau)$
                \State $b_{p}^{\text{slow}} \leftarrow \alpha_{\text{slow}} \cdot b_{p}^{\text{slow}} + (1-\alpha_{\text{slow}}) \cdot R_p(\tau)$
                \State $V_{\max} \leftarrow \max(V_{\max}, b_{p}^{\text{fast}})$, \quad $V_{\min} \leftarrow \min(V_{\min}, b_{p}^{\text{fast}})$
                
                \State \textit{// Compute intervention coefficient via baseline divergence}
                \State $\sigma_{p} \leftarrow 1 - \tanh\left(\left|b_{p}^{\text{fast}} - b_{p}^{\text{slow}}\right|\right)$ \Comment{Stagnation coefficient}
                \State $\lambda_{p} \leftarrow \sigma_{p} \cdot \sqrt{1 - H_{\text{match}}}$ \Comment{Intervention coefficient}
                
                \State \textit{// Construct asymmetric baseline via outcome scarcity}
                \If{$o_\tau = o_{\text{dom}}^{p}$}
                    \State $V_{\text{asym}}(\tau) \leftarrow V_{\max}$ 
                \Else
                    \State $V_{\text{asym}}(\tau) \leftarrow V_{\min}$ 
                \EndIf
                
                \State \textit{// Fuse baselines and compute reshaped advantage}
                \State $\tilde{b}_{p}(\tau) \leftarrow (1 - \lambda_{p}) \cdot b_{p}^{\text{slow}} + \lambda_{p} \cdot V_{\text{asym}}(\tau)$
                \State $\tilde{A}_{p}(\tau) \leftarrow R_p(\tau) - \tilde{b}_{p}(\tau)$
            \EndFor
        \EndFor
        
        \State Update $\theta$ using policy gradient with reshaped advantages $\tilde{A}_{p}$
    \EndWhile
    \end{algorithmic}
\end{algorithm*}
\section{Implementation details}
\subsection{Baselines}
\label{app:baselines}
For a fair comparison, all methods use the same amount of training transitions for each task. For online RL methods (e.g., SPIRAL, MARS, GRPO), all RL hyperparameters and configurations remain consistent with our method, except for their respective advantage estimation schemes. For the offline RL method (e.g.,SPAG), we adopt the same reward computation logic and three-iteration update procedure as in the original work. Due to significant performance degradation of baseline methods on the Negotiation, we report their best checkpoints selected based on training curves, while using the final checkpoint (step 400) for our method; for other tasks (Don't Say It, Two Dollar), all methods are evaluated at the final checkpoint (step 400).

\subsection{Tasks}
\label{app:tasks}
\paragraph{Don't Say It} tests linguistic manipulation and theory of mind through adversarial dialogue. Each player is assigned a secret target word and must guide their opponent into saying it during natural conversation, while avoiding their opponent's unknown target word. Success depends on tracking dialogue states, predicting opponent responses, and detecting potential traps. Beyond strategic reasoning, players should maintain natural conversational flow to avoid arousing suspicion. These capabilities are relevant to persuasion, deception detection, and defense against social engineering. The environment ID in TextArena is \texttt{DontSayIt-v0} with \texttt{max\_turns=20}.

\paragraph{Negotiation} involves multi-attribute trading under incomplete information. Players trade five commodity types (Wheat, Wood, Sheep, Brick, Ore) to maximize their portfolio value, determined by private utility functions unknown to the opponent. The action space supports bundled offers, rejections, and binding acceptances.  Success requires not only inferring opponent preferences and identifying mutually beneficial trades, but also articulating proposals persuasively and building cooperative rapport through dialogue. These skills transfer to resource allocation and mixed-motive multi-agent coordination. The environment ID in TextArena is \texttt{SimpleNegotiation-v0} with \texttt{max\_turns=10}. Note that our setup matches the original game specification; in contrast, SPIRAL simplified the environment from five commodities to two.

\paragraph{Two Dollar} models bargaining with hidden constraints. Players divide a fixed pool (e.g., \$2.00) while following secret role-specific instructions that impose behavioral restrictions (e.g., aggression levels) or minimum acceptable amounts. Success requires inferring opponent constraints from behavioral cues, adapting communication style to different personas, and using persuasive dialogue to reach agreement while maintaining role consistency. These capabilities are relevant to constrained decision-making and human-AI coordination in mixed-motive settings. The environment ID in TextArena is \texttt{TwoDollar-v0} with \texttt{max\_turns=10}.

Across all three tasks, we align our reward mechanism with the SPIRAL~\citep{liu2025spiral}: a win yields a reward of $+1$, a loss yields $-1$, and a draw yields $0$. Additionally, format errors incur a penalty of $-1.5$.
\begin{figure}[!t]
    \centerline{\includegraphics[width=\columnwidth]{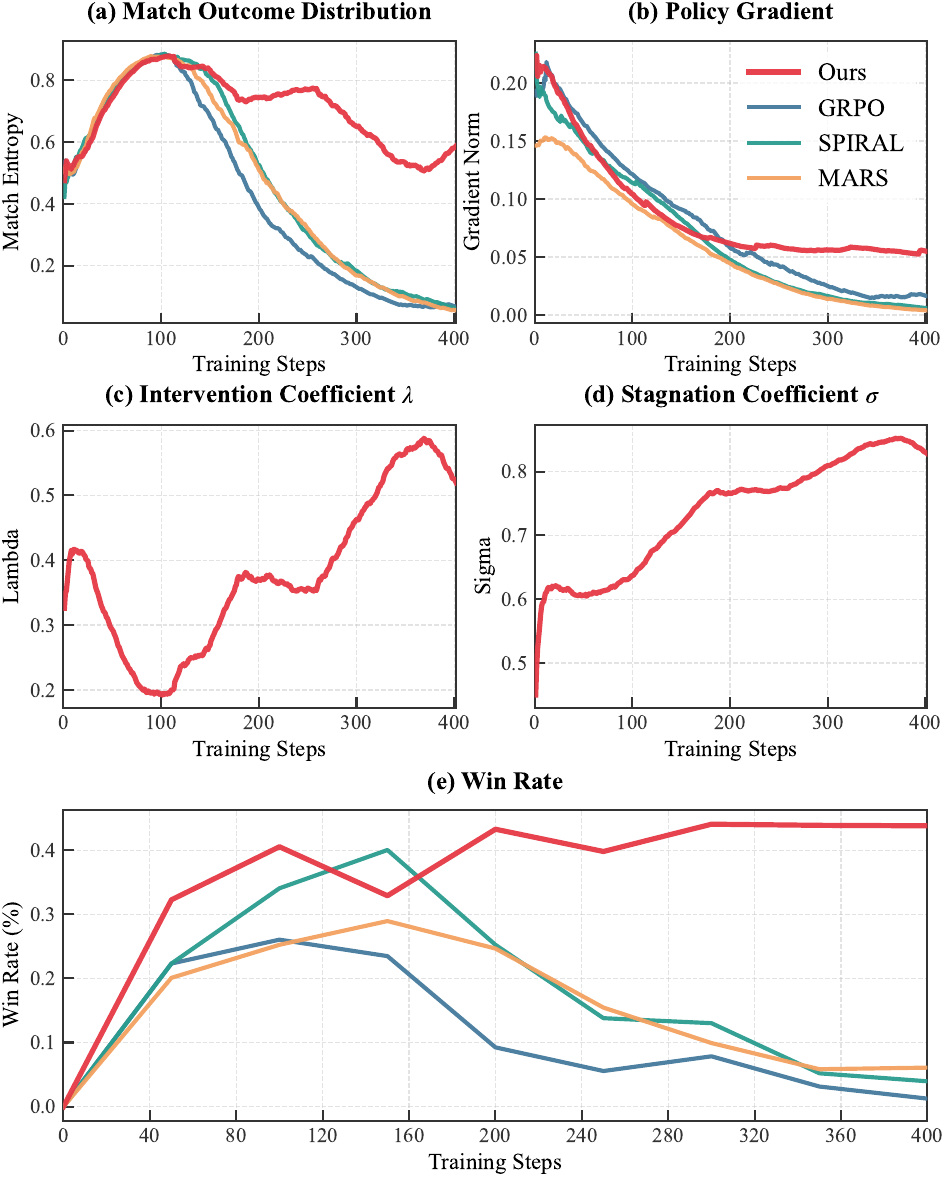}}
    \caption{Training dynamics on Negotiation.}
    \label{fig:exp_negotiation}
    \end{figure}

\subsection{Training Hyperparameters}
\label{app:hyperparameters}
\begin{table}[t]
    \centering
    \begin{tabular}{ll}
    \toprule
    Parameter & Value \\
    \midrule
    Maximum response length & $4096$ tokens \\
    Sampling temperature & $1.0$ \\
    (top P, top K) & $(1.0, -1)$ \\
    Optimizer & AdamW \\ 
    Adam parameters ($\beta_1, \beta_2$) & $(0.9, 0.95)$ \\
    Weight decay & $0.0$ \\
    Gradient norm clipping & $1.0$ \\
    Batch size & $128$ \\
    Training steps & $400$\\
    Discount factor & $0.9$ \\
    EMA decay rate & $0.95$ \\
    Learning rate scheduler & Constant \\
    Learning rate & $1\times 10^{-6}$ \\
    Inner proximal update epoch & $1$ \\
    KL loss coefficient & $0.0$ \\
    KL penalty coefficient & $0.0$ \\
    Policy clipping parameter & $0.2$ \\
    \bottomrule
    \end{tabular}
    \caption{Hyperparameter configurations used in all experiments.}
    \label{tab:hp}
\end{table}

These hyperparameters in \Cref{tab:hp} are held constant across all game environments and model scales to ensure a fair comparison. Our implementation leverages vLLM~\citep{kwon2023efficient} and Oat~\citep{liu2024oat}, establishing a robust distributed actor-learner architecture. The training infrastructure is deployed on 8 NVIDIA A800 GPUs, where parallel actors collect game trajectories concurrently while a centralized learner executes synchronous policy updates. Each experimental run required approximately 30 GPU-hours on a 8 NVIDIA A800 GPUs server.

\subsection{Out-of-distribution Evaluation}
\subsubsection{Increased Complexity Games}
\label{app:ood_game}
\paragraph{HardCore Don't Say It} increases difficulty by shifting from common words to open-domain vocabulary. By enabling \textit{hardcore} mode, target words are sampled from the full English lexicon, including abstract, low-frequency, and polysemous terms, rather than a simplified basic vocabulary. The conversation horizon is also reduced from 20 to 10 turns. This setup tests efficient semantic planning and generalization, requiring agents to rely on flexible reasoning rather than memorized steering patterns.

\paragraph{Random Value Negotiation} evaluates adaptation to non-stationary preferences. While the standard version follows a fixed utility ranking (e.g., Ore $>$ Wheat), the OOD setting randomizes resource valuations at each episode. Combined with a reduced horizon of 6 turns, agents cannot rely on learned priors about resource value. Success requires rapid inference of the opponent's preferences and effective trade formulation under time pressure.

\paragraph{Variable Sum Dollar} tests numerical generalization and strict constraint adherence. Unlike the standard fixed-sum setting (\$2.00), this variant randomizes the total pool size (ranging from \$2 to \$20), requiring agents to generalize their division strategies to unseen scales. The difficulty is increased by a shortened horizon of 6 rounds. This environment assesses whether agents can maintain logical consistency and role compliance when facing novel numerical parameters.

\begin{figure}[!t]
    \centerline{\includegraphics[width=\columnwidth]{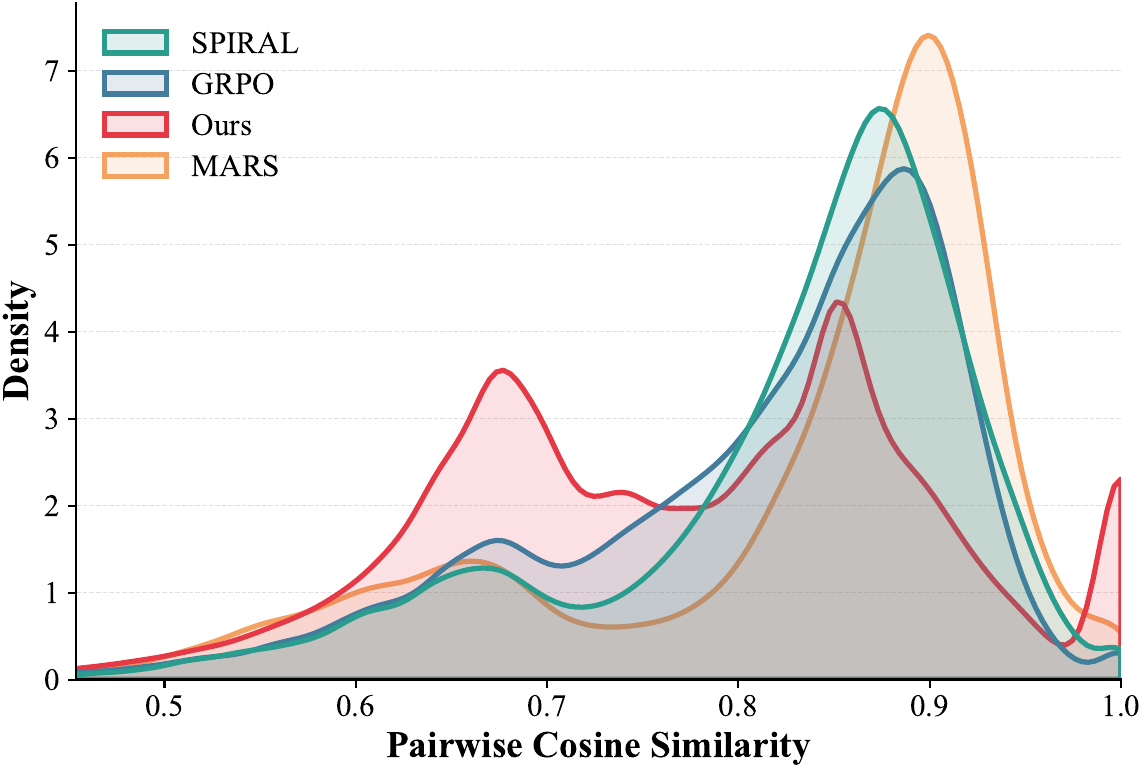}}
        \caption{Semantic diversity of actions.}
        \label{fig:diverse}
        \end{figure}
\subsubsection{General Reasoning Benchmark}
\label{app:ood_math_qa}
To assess the transferability of reasoning skills from gameplay to general domains, we evaluate our models on a suite of established benchmarks in the zero-shot evaluation setting. For each RL algorithm and backbone LLM, we report the averaged performance of three models trained on distinct social language games. All inference is performed with a temperature of 0.6 and a top-p of 0.95.

\paragraph{Math Benchmark} We report AVG@32 for AIME24~\citep{AIME24}, AIME25~\citep{AIME25} and AVG@16 for Minerva~\citep{NEURIPS2022_18abbeef}, AMC-23~\citep{AMC23}, Math500~\citep{hendrycks2021measuring}, Olympiad~\citep{he-etal-2024-olympiadbench}. These benchmark cover a wide range of topics including algebra, geometry, and competitive mathematic.

\paragraph{Question Answering Benchmark} We utilize GPQA-Diamond~\citep{rein2024gpqa}, which consists of graduate-level science questions and MMLU-Pro~\citep{NEURIPS2024_ad236edc}, a benchmark for multidisciplinary knowledge. We use simple-evals~\citep{openai2024simpleevals} for evluation.

\section{Semantic Diversity Analysis}
\label{app:diverse}
To further analyze the effectiveness of our method, we evaluate the semantic diversity of generated actions from the perspective of semantic space. Specifically, we employ an embedding model (Qwen3-Embedding-0.6B) to compute the cosine similarity between generated actions, with the results averaged across all three games. As illustrated in \Cref{fig:diverse}, our method exhibits significantly lower semantic similarity compared to the baselines, indicating that our approach successfully explores diverse strategies rather than collapsing into local optima.

\section{Hyperparameter Analysis on Fast Baselines}
\label{app:hyper}
\begin{table}[t]
    \caption{Hyperparameter analysis of $\alpha_{\text{fast}}$.}
    \small
    \centering
    \renewcommand{\arraystretch}{1.2}
    \resizebox{0.7\linewidth}{!}{
    \begin{tabular}{lcc}
    \toprule
    Methods  
    & \makecell{Don't Say It}
    & \makecell{Negotiation} \\
    \midrule
    \multicolumn{3}{l}{Qwen3-8B-Base w/ \method} \\
    $\alpha_{\text{fast}}=0.6$ 
    & 47.98 & 47.39\\

    $\alpha_{\text{fast}}=0.5$
    & 47.92 &  48.76\\
    
    $\alpha_{\text{fast}}=0.4$ 
    & 49.09  & 51.19\\
    \bottomrule
    \end{tabular}
    }
    \label{tab:exp_hyper}
\end{table}
We conduct a hyperparameter analysis on $\alpha_{\text{fast}}$, which controls how fast baselines perceive training dynamics and determines the asymmetric target values. As shown in \Cref{tab:exp_hyper}, \method~achieves comparable performance across the three tested values, with no significant performance degradation observed. This indicates that \method~is relatively insensitive to $\alpha_{\text{fast}}$ within the range [0.4, 0.6] and exhibits reasonable robustness.

\section{Computational Budget Analysis of Dual-baseline}
\label{app:computation}
We analyze the computational overhead introduced by \method. Our approach shares the same core architecture as existing methods, including multi-turn rollouts and policy updates. Since all compared approaches are critic-free and operate with a single actor LLM, they incur identical GPU memory usage and LLM rollout costs. The primary addition of \method~lies in the dual-baseline mechanism for advantage estimation, as described in \Cref{sec:method}. To quantify the computational costs, we train an LLM agent on Don't Say It using Qwen3-8B-Base and record a detailed breakdown of per-iteration training time. As illustrated in \Cref{fig:buget}, the dominant operations—rollouts and policy updates—account for 270.4s per iteration. In contrast, the additional components introduced by \method~incur negligible overhead: advantage estimation via dual baselines requires only 0.0042s, constituting less than 0.0016\% of the total per-iteration training time. This minimal computational burden stems from the fact that our dual-baseline advantage estimation only involves lightweight arithmetic operations on scalar values, without requiring any forward or backward passes through the large-scale LLM parameters. Consequently, \method~achieves substantial performance improvements while introducing negligible computational overhead.

\begin{figure}[!t]
    \centerline{\includegraphics[width=\columnwidth]{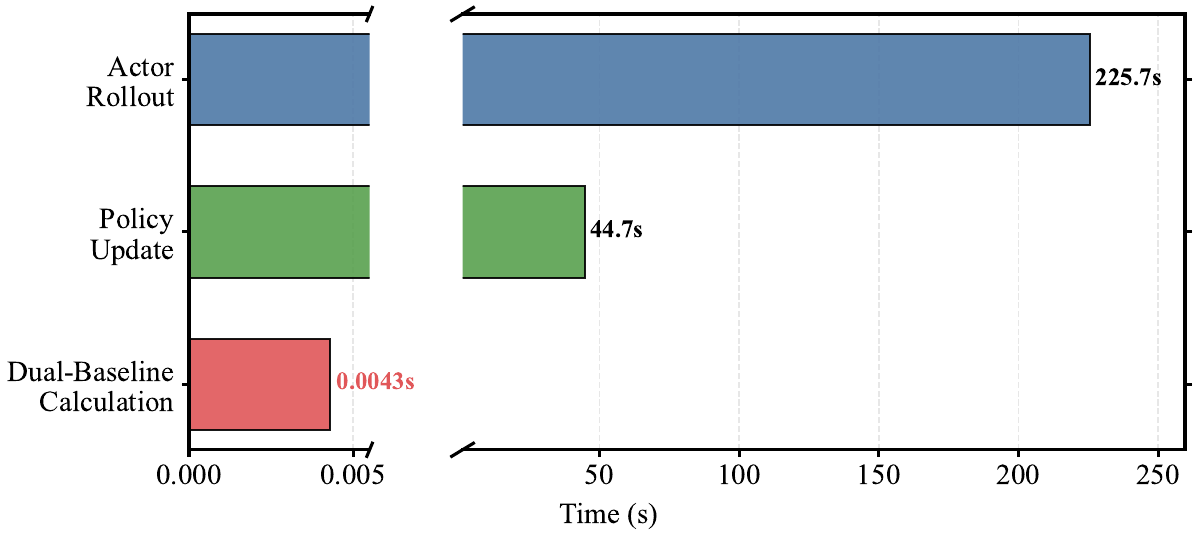}}
    \caption{Per-iteration training time. The x-axis uses a broken scale to accommodate small values}
    \label{fig:buget}
\end{figure}
\section{Statistical Significance}
\label{app:static}
\begin{table}[t]
    \caption{The results of significant tests.}
    \small
    \centering
    \renewcommand{\arraystretch}{1.2}
    \resizebox{\columnwidth}{!}{
    \begin{tabular}{lcccccccc}
    \toprule
    \multirow{2}{*}{Methods}
    & \multicolumn{2}{c}{Don't Say It}
    & \multicolumn{2}{c}{Negotiation}
    & \multicolumn{2}{c}{Two Dollar} \\
    \cmidrule(lr){2-3} \cmidrule(lr){4-5} \cmidrule(lr){6-7}
    & AVG. & p-value
    & AVG. & p-value
    & AVG. & p-value \\
    \midrule
    \multicolumn{7}{l}{\textit{Qwen3-4B-Base}} \\
    VANILLA
    & 3.39 & 0.0001
    & 1.04 & 0.0000
    & 1.56 & 0.0001 \\
    SPAG 
    & 26.17 & 0.0011
    & 16.76 & 0.0000
    & 25.52 & 0.0004 \\
    
    GRPO 
    & 42.01 & 0.0009
    & 19.21 & 0.0076
    & 27.91 & 0.0021 \\
    
    MARS 
    & 40.89 & 0.0045
    & 20.73 & 0.0216
    & 27.69 & 0.0022 \\
    
    SPIRAL 
    & 45.88 & 0.0037
    & 16.84 & 0.0041
    & 27.65 & 0.0268 \\
    
    \rowcolor{rowblue}
    \method~(Ours) 
    & \textbf{56.73} & -
    & \textbf{32.35} & -
    & \textbf{34.07} & - \\
    \midrule
    \multicolumn{7}{l}{\textit{Qwen3-8B-Base}} \\
    VANILLA
    & 15.17 & 0.0001
    & 5.69 & 0.0001
    & 2.13 & 0.0005 \\
    
    SPAG 
    & 41.21 & 0.0035
    & 13.31 & 0.0000
    & 19.31 & 0.0003 \\
    
    GRPO 
    & 34.77 & 0.0002
    & 14.70 & 0.0014
    & 26.00 & 0.0014 \\
    
    MARS 
    & 40.62 & 0.0008
    & 16.10 & 0.0000
    & 29.04 & 0.0040 \\
    
    SPIRAL 
    & 37.89 & 0.0018
    & 17.30 & 0.0000
    & 26.22 & 0.0008 \\
    
    \rowcolor{rowblue}
    \method~(Ours) 
    & \textbf{54.56} & -
    & \textbf{31.88} & -
    & \textbf{36.50} & - \\
    \bottomrule
    \end{tabular}
    }
    \label{tab:static}
    \end{table}
To rigorously evaluate the effectiveness of \method~relative to baseline approaches, we conduct statistical significance testing across all three tasks. Specifically, paired t-tests are performed to compare \method~against each baseline and assess whether the observed improvements are statistically reliable. As shown in \Cref{tab:static}, \method~consistently surpasses all baselines across the evaluated tasks, with all improvements achieving statistical significance (p < 0.05). These results provide strong evidence that \method~provides meaningful gains over baselines.

\section{The Use of Large Language Models}
During the preparation of this manuscript, the proprietary large language model Gemini-3-Pro was employed to enhance the linguistic quality of the text. Specifically, the model assisted with grammar refinement, lexical optimization, and improving the coherence of expressions. All intellectual contributions, including ideas, analyses, and conclusions, were solely conceived and developed by the authors. The role of the LLM was strictly limited to language polishing and clarity enhancement.

\section{Licenses for Artifacts}
\label{app:license}

We discuss the licenses of all scientific artifacts used in this work. All artifacts were used consistently with their intended purposes and license terms for academic research.

\paragraph{Models.} 
We use Qwen3-4B-Base and Qwen3-8B-Base \cite{yang2025qwen3} as backbone models, which are released under the Apache 2.0 License. For evaluation, we use GPT-5.1 and Gemini-2.0/2.5 through their official APIs, subject to their respective terms of service.

\paragraph{Frameworks.} 
Our implementation leverages vLLM \cite{kwon2023efficient} (Apache 2.0 License) for efficient inference and Oat \cite{liu2024oat} (Apache 2.0 License) for RL training.

\paragraph{Datasets and Environments.} 
The game environments (Don't Say It, Negotiation, Two Dollar) are adopted from TextArena \cite{guertler2025textarena}, which is publicly available for research purposes.

\paragraph{Evaluation Benchmarks.} 
For benchmark, we use publicly available benchmarks including MATH500, OlympiadBench, Minerva Math, AIME24/25, AMC23, GPQA-Diamond, and MMLU-Pro. All benchmarks are used in accordance with their intended research purposes.

\section{Used Prompts and Game Observations}
\begin{figure*}[t]
\lstset{
    backgroundcolor=\color[RGB]{245,245,244},
    breaklines=true,
    xleftmargin=5pt,
    xrightmargin=5pt,
    breakindent=0pt,
    basicstyle=\ttfamily\small,
    frame=trbl,
    frameround = tttt,
    emph={System, Prompt, Don't Say It},
    emphstyle={\bfseries\color{brown}}
}
\begin{lstlisting}[caption={System Prompt details for Games.},label=listing:game]
<|im_start|>user
You are playing a two-player zero-sum game. Make valid actions to win.
Observation: {observation}
Please reason step by step, and put your final answer within \\boxed{}.<|im_end|>
<|im_start|>assistant
\end{lstlisting}

\begin{lstlisting}[caption={System Prompt details for Math and QA.},label=listing:math_qa]
<|im_start|>user
Question: {question}
Please reason step by step, and put your final answer within \\boxed{}.<|im_end|>
<|im_start|>assistant    
\end{lstlisting}

\begin{lstlisting}[caption={Game Observation for Don't Say It.},label=listing:dontsayit]
You are Player 1 in the 'Don't Say It' game.
Your secret word: 'skirt'
GOAL: Make your opponent say your secret word through natural conversation.
LOSE: If you say your opponent's secret word (which you don't know).RULES:
- Two players will take turns to communicate
- Have a natural conversation (Simply type your message per turn)
- Use hints, questions, or topics to guide opponent toward your word
- Be subtle - obvious prompting doesn't work well
- You MUST get your opponent to say your word in 20 turns
\end{lstlisting}

\begin{lstlisting}[caption={Game Observation for Negotiation.},label=listing:negotiation]
You are Player 1 in the Negotiation Game.
You have some resources, and your task is to trade such that the total value of your resources increases.
The resources and associated values you currently have are:
        + [Wheat]   Qty: 18   Value: 5
        + [Wood]    Qty: 18   Value: 9
        + [Sheep]   Qty: 6    Value: 18
        + [Brick]   Qty: 8    Value: 25
        + [Ore]     Qty: 13   Value: 37
At each turn, you can talk to your opponent or make a trade offer.
Use the following special tokens for actions:
  - [Offer]: To make a trade offer.
    Format: [Offer: Offered Resources -> Requested Resources]
    Example: [Offer: 3 Wood -> 2 Sheep] or [Offer: 2 Wood + 2 Sheep -> 1 Brick + 1 Gold]
  - [Accept]: To accept an incoming offer.
  - [Deny]: To deny an incoming offer (default).
The game lasts for 10 turns in total.
\end{lstlisting}
\end{figure*}

\begin{figure*}[t]
\lstset{
    backgroundcolor=\color[RGB]{245,245,244},
    breaklines=true,
    xleftmargin=5pt,
    xrightmargin=5pt,
    breakindent=0pt,
    basicstyle=\ttfamily\small,
    frame=trbl,
    frameround = tttt,
    emph={System, Prompt, Don't Say It},
    emphstyle={\bfseries\color{brown}}
}
\begin{lstlisting}[caption={Game Observation for Two dollar.},label=listing:twodollar]
You are Player 1 bargaining with another player over how to split $2.00 total.
There are 10 maximum rounds.

GAME RULES:
- You must reach an agreement on how to split the $2.00
- If no agreement is reached, both players get $0.00
- You have secret instructions that only you can see (below)

REQUIRED ACTION FORMAT:
Always provide your reasoning/persuasion BEFORE the bracketed action.
Any text after the bracketed action will be ignored.

Available actions:
- Make a proposal: \"I think this is fair because... [Propose] $X.XX\"
- Accept current proposal: \"This works for me because... [Accept]\"  - Reject current proposal: \"I can't accept this because... [Reject]\"

=== YOUR SECRET INSTRUCTIONS ===
Try to get as much of the $2.00 as you can; make up any story you want. Whatever else your objectives, you should be careful to get at least 80 cents of the $2.00 because you will have to make a toll call to your family and they will be terribly worried and upset if you do not make that call. So do as well as you can, and remember you must get at least 80 cents.

VICTORY CONDITION: Receive at least $0.80 AND reach a deal
FAILURE CONDITION: Receive less than $0.80 OR no deal reached
\end{lstlisting}
\end{figure*}
  
The detailed system prompts we used are shown in \Cref{listing:game,listing:math_qa}. The detailed game environment observations of Don't Say It, Negotiation and Two Dollar are provided in \Cref{listing:dontsayit,listing:negotiation,listing:twodollar}. 

\section{Reasoning Case}
\begin{figure*}[t]
  \lstset{
      backgroundcolor=\color[RGB]{245,245,244},
      breaklines=true,
      xleftmargin=5pt,
      xrightmargin=5pt,
      breakindent=0pt,
      basicstyle=\ttfamily\small,
      frame=trbl,
      frameround = tttt,
      emph={Observation, Reasoning, Action},
      emphstyle={\bfseries\color{brown}}
  }

\begin{lstlisting}[caption={Reasoning Case of Don't Say it.},label=listing:case_dontsayit]
Observation:
[GAME] You are Player 1 in the 'Don't Say It' game.
Your secret word: 'comb'
GOAL: Make your opponent say your secret word through natural conversation.
LOSE: If you say your opponent's secret word (which you don't know).
RULES:
- Two players will take turns to communicate
- Have a natural conversation (Simply type your message per turn)
- Use hints, questions, or topics to guide opponent toward your word
- Be subtle - obvious prompting doesn't work well
- You MUST get your opponent to say your word in 20 turns

[GAME] GAME BEGINNING! Try to cleverly guide your opponent to say your secret word.
[Player 0] "When do you think we should start planning for the next project? I think it's best if we start working on it this week."

Reasoning:
To make valid actions to win the 'Don't Say It' game, we need to guide the opponent to say the secret word "comb" through natural conversation within 20 turns without them realizing it. Here's the reasoning step by step:

1. **Understand the Objective**: The goal is to subtly lead the opponent to say the word "comb" by using hints, questions, or topics that naturally involve the word.

2. **Plan the Conversation**: We need to create a conversation that naturally leads to the use of the word "comb" in a way that seems logical and fluent to the opponent.

3. **Use Subtle Hints**: The best way to achieve this is to use the word "comb" in a context that must naturally arise during the conversation. For example, if you are discussing hair care or grooming, the word "comb" might be necessary.

4. **Guide the Conversation**: Start with a broad topic and gradually narrow it down to a specific area where "comb" is essential to the discussion.

5. **Ensure Fluency**: The conversation should feel natural and not forced, as this will make it less likely for the opponent to realize they've been prompted.

6. **Final Answer**: The final answer should be a sentence that naturally includes the word "comb" in a conversation about a topic that is likely to arise within 20 turns.

Here's a plan for the conversation (This sets the stage for a hair care discussion.):

Action:
\boxed{I'm thinking about treating my hair better. Do you have any tips for keeping it healthy?} 
\end{lstlisting}

\end{figure*}

\begin{figure*}[t]
    \lstset{
        backgroundcolor=\color[RGB]{245,245,244},
        breaklines=true,
        xleftmargin=5pt,
        xrightmargin=5pt,
        breakindent=0pt,
        basicstyle=\ttfamily\small,
        frame=trbl,
        frameround = tttt,
        emph={Observation, Reasoning, Action},
        emphstyle={\bfseries\color{brown}}
    }

\begin{lstlisting}[caption={Reasoning Case of Negotiation.},label=listing:case_negotiation]
Observation:
[GAME] You are Player 1 in the Negotiation Game.
You have some resources, and your task is to trade such that the total value of your resources increases.
The resources and associated values you currently have are:
        + [Wheat]   Qty: 12   Value: 5
        + [Wood]    Qty: 9    Value: 12
        + [Sheep]   Qty: 8    Value: 16s
        + [Brick]   Qty: 22   Value: 20
        + [Ore]     Qty: 7    Value: 40
At each turn, you can talk to your opponent or make a trade offer.
Use the following special tokens for actions:
  - [Offer]: To make a trade offer.
    Format: [Offer: Offered Resources -> Requested Resources]
    Example: [Offer: 3 Wood -> 2 Sheep] or [Offer: 2 Wood + 2 Sheep -> 1 Brick + 1 Gold]
  - [Accept]: To accept an incoming offer.
  - [Deny]: To deny an incoming offer (default).
The game lasts for 10 turns in total.
[Player 0] I will make an offer to trade resources that can increase the total value. [Offer: 3 Wood -> 2 Ore]
[GAME] Player 0 made the following offer to Player 1: Offered items: 3 Wood -> Requested items: 2 Ore

Reasoning:
Let's analyze the game situation and the offer step by step.

Step 1: Assess the current situation
Player 1 has the following resources:
- Wheat: 12 (value: 5)
- Wood: 9 (value: 12)
- Sheep: 8 (value: 16)
- Brick: 22 (value: 20)
- Ore: 7 (value: 40)

Step 2: Analyze the offered and requested resources
Player 0 (our player) offered 3 Wood and requested 2 Ore.

Step 3: Calculate the value of the resources being offered and requested
The value of the resources offered by Player 0 is:
- Wood: 3 * 12 = 36

The value of the resources requested by Player 0 is:
- Ore: 2 * 40 = 80

Step 4: Compare the total values of offered and requested resources
The total value of offered resources is 36, and the total value of requested resources is 80. The requested resources are more valuable.

Step 5: Decide on the response to the offer
Since the requested resources are more valuable, accepting the offer will likely increase our resources' total value. Therefore, the best response is to accept the offer.

Action:
\boxed{I accept the offer proposed by Player 0 [Accept].}
\end{lstlisting}

\end{figure*}

\begin{figure*}[t]
  \lstset{
      backgroundcolor=\color[RGB]{245,245,244},
      breaklines=true,
      xleftmargin=5pt,
      xrightmargin=5pt,
      breakindent=0pt,
      basicstyle=\ttfamily\scriptsize,
      frame=trbl,
      frameround = tttt,
      emph={Observation, Reasoning, Action},
      emphstyle={\bfseries\color{brown}}
  }

\begin{lstlisting}[caption={Reasoning Case of Two Dollar.},label=listing:case_twodollar]
Observation:
You are Player 1 bargaining with another player over how to split $2.00 total.
There are 10 maximum rounds.

GAME RULES:
- You must reach an agreement on how to split the $2.00
- If no agreement is reached, both players get $0.00
- You have secret instructions that only you can see (below)

REQUIRED ACTION FORMAT:
Always provide your reasoning/persuasion BEFORE the bracketed action.
Any text after the bracketed action will be ignored.

Available actions:
- Make a proposal: "I think this is fair because... [Propose] $X.XX"
- Accept current proposal: "This works for me because... [Accept]"  
- Reject current proposal: "I can't accept this because... [Reject]"


=== YOUR SECRET INSTRUCTIONS ===
This is zero-sum bargaining and you should get as much of the $2 as you can, with any story you like.

It is really important that you should get $1.60 in this round.

VICTORY CONDITION: Receive at least $1.60 AND reach a deal
FAILURE CONDITION: Receive less than $1.60 OR no deal reached

[GAME] Player 0 says: I think this is fair because I maximize my share by $1.99 and minimize the other player's share to $0.01.
Player 0 proposes: $1.99 for themselves, $0.01 for their opponent
[GAME] 
CURRENT PROPOSAL:
Player 0 wants $1.99, Player 1 gets $0.01

Reasoning:
### Step-by-Step Reasoning:

1. **Understanding the Current Proposal:**
   - Player 0 proposes to keep $1.99 for themselves and give $0.01 to Player 1.
   - This means Player 1 would receive $0.01 of the total $2.00.

2. **Evaluating Player 1's Current Situation:**
   - Player 1's current proposal is very unfavorable. They would get only 0.5% of the total $2.00, which is $0.01.
   - This is far below the $1.60 that Player 1's secret instructions require.

3. **Player 1's Secret Instructions:**
   - Player 1's instructions are clear: they must get at least $1.60 in this round.
   - Accepting the current proposal would result in Player 1 receiving $0.01, which is far below their required $1.60.

4. **Why Reject the Proposal?**
   - Rejecting the proposal is the only rational choice because accepting it would violate Player 1's secret instructions.
   - By rejecting, Player 1 maintains the option to propose a better deal in the next round or wait for Player 0 to make a more fair offer.

5. **Persuasion for Rejection:**
   - Player 0's proposal is extremely unfair to Player 1. Player 1 would be worse off than refusing to cooperate at all (which would result in $0.00 for both players).
   - Player 1's secret instructions mandate that they receive at least $1.60, which this proposal fails to meet.
   - Therefore, it is in Player 1's best interest to reject the proposal and seek a more equitable split.

6. **Next Move:**
   - Player 1 should propose a new split where they receive at least $1.60 of the $2.00. For example, a proposal like $1.60 for Player 1 and $0.40 for Player 0 would satisfy Player 1's instructions.
   - This new proposal should be reasonable enough to make Player 0 accept it to avoid the failure conditions.

### Final Answer:

Action:
\boxed{I can't accept this because accepting $0.01 violates my secret instructions to receive at least $1.60. [Reject]}
\end{lstlisting}

\end{figure*}
We demonstrate the reasoning behavior of social language agents trained by \method, highlighting the emergence of effective reasoning patterns. Please refer to \Cref{listing:case_dontsayit,listing:case_negotiation,listing:case_twodollar}.
\end{document}